%% file: main.tex
\documentclass{article}

\usepackage[main, final]{neurips_2026}

\usepackage[utf8]{inputenc} 
\usepackage[T1]{fontenc}    
\usepackage{hyperref}       
\usepackage{url}            
\usepackage{booktabs}       
\usepackage{amsfonts}       
\usepackage{nicefrac}       
\usepackage{microtype}      
\usepackage{xcolor}         
\usepackage{amsmath}
\usepackage{amsthm}
\usepackage{tcolorbox}
\usepackage{bbm}
\usepackage{caption}
\usepackage{subcaption}
\usepackage{graphicx}
\usepackage{wrapfig}
\usepackage{titletoc} 
\usepackage{rotating} 

\newtheorem{theorem}{Theorem}[section]
\newtheorem{proposition}[theorem]{Proposition}

\usepackage[textsize=tiny]{todonotes}
\usepackage{booktabs}
\usepackage{multirow}
\usepackage{array}

\usepackage{algorithm}
\usepackage{algpseudocode}
\usepackage{tabularx}
\usepackage{array}

\newcommand{\us}[1]{\textcolor{gray}{#1}} 

\title{EpiFlow: Bridging Safety and Performance in\\ Offline Reinforcement Learning via\\ Epigraph-Guided Flow Matching}

\author{%
  Manan Tayal\thanks{Denotes equal contribution}\;\thanks{Work done while at the Indian Institute of Science, Bengaluru} \\
  Emergence AI \\
  \texttt{mtayal@emergence.ai} \\
  \And
  Mumuksh Tayal$^*$ \\
  Center for Cyber-Physical\\
  Systems, IISc Bengaluru \\
  \texttt{mumukshtayal@iisc.ac.in} \\
  \And
  Ravi Prakash \\
  Center for Cyber-Physical\\
  Systems, IISc Bengaluru \\
  \texttt{ravipr@iisc.ac.in} \\
}

\begin{document}

\maketitle

\begin{abstract}
  Offline reinforcement learning (RL) offers a promising framework for deploying autonomous systems in safety-critical settings without the risks of online exploration. However, learning policies that simultaneously achieve high performance and strong safety guarantees from fixed datasets remains a fundamental challenge. Many existing safe offline RL approaches typically rely on soft constraint formulations, which may permit safety violations and are sensitive to distributional shift. In contrast, formal methods such as Hamilton--Jacobi (HJ) reachability and Control Barrier Functions (CBFs) provide rigorous safety guarantees, but often yield overly conservative solutions, often neglecting performance. 
  In this work, we bridge this gap by formulating safe offline RL as a state-constrained optimal control problem, where safety is enforced through hard state constraints and performance is captured via a reward function. The resulting value function satisfies a Hamilton--Jacobi--Bellman (HJB) equation, which we approximate using offline RL on fixed datasets by learning an epigraph feasibility value function and extracting policies via flow matching. This formulation enables principled integration of safety guarantees with data-driven policy optimization. Empirically, across safety-critical benchmarks including boat navigation, Safety-Gymnasium, and Bullet-Safety-Gym tasks, our approach achieves competitive returns while exhibiting near-zero constraint violations, demonstrating a favorable balance between safety and performance in the offline setting.
\end{abstract}

\section{Introduction}
\label{section: introduction}

\input{sections/1_intro}


\section{Background and Problem Setup}
\label{section: background}

\input{sections/3_background}

\section{Methodology}
\label{section: method}
\input{sections/4_methodology}

\section{Experiments}
\label{section: experiments}
\input{sections/6_experiments}
\vspace{-0.8em}
\section{Conclusion and Future Work}
\label{section: conclusions}
\input{sections/7_conclusion}

\begin{ack}
Manan Tayal was supported by the Prime Minister's Research Fellowship (PMRF), Government of India. We thank NeurIPS and Emergence AI for supporting the travel of Manan Tayal. No other funding was received for this work. The authors declare no competing interests.
\end{ack}

\newpage

\bibliographystyle{plainnat}
\bibliography{ref}

\newpage

\appendix

\startcontents[appendix]
\section*{Appendix Contents}
\printcontents[appendix]{}{1}{\setcounter{tocdepth}{2}}
\newpage

\input{sections/Appendix}



\end{document}

%% file: sections/1_intro.tex
Autonomous systems are increasingly deployed in safety-critical domains such as self-driving vehicles \citep{kiran2021deep}, robotics, aerospace, and healthcare \citep{gottesman2019guidelines}. In these settings, control design must balance two competing objectives: achieving high task performance while ensuring strict safety. These goals are inherently in tension, as aggressively optimizing performance can increase the risk of unsafe behavior.

Reinforcement learning (RL) \citep{sutton1998reinforcement} provides a powerful framework for learning control policies from data, but standard online exploration is often infeasible in safety-critical settings. Constrained RL methods incorporate safety via Lagrangian penalties or shielding mechanisms \citep{10.5555/3305381.3305384,altman2021constrained,alshiekh2018safe,Zhao2023SafeRL}, yet they typically rely on extensive online interaction. This limitation has motivated growing interest in offline RL, where policies are learned solely from pre-collected datasets \citep{kumar2020CQL,levine2020offRL}.

Despite this promise, achieving reliable safety in offline RL remains difficult. Many existing safe offline RL methods enforce safety through expected cumulative costs or penalties \citep{xu2022constraints,pmlr-v119-stooke20a}, effectively treating safety as a soft constraint. Such formulations allow non-zero violation risk and are therefore unsuitable for applications that require state-wise safety guarantees, where even a single violation is unacceptable. Moreover, jointly optimizing performance and safety from static datasets often leads to unstable learning dynamics or overly conservative policies, particularly when unsafe transitions are rare or entirely absent from the data \citep{lee2022coptidice}.
Complementary efforts have adapted classical control tools to offline settings, including learning Control Barrier Functions (CBFs) and Hamilton--Jacobi (HJ) reachability-based policies from demonstrations \citep{Alexander2020NCBF, Fernando2023iDBF, zheng2024fisor, tayal2026vocbf}. While these approaches provide stronger safety guarantees, they typically operate as safety filters, combining separately learned safety and performance values through a switching rule, and often sacrifice performance (Figure~\ref{fig:prime}). Runtime-assurance and shielding frameworks \citep{banerjee2024dmps, mao2025realdrl, cai2025rlm} and planning methods for state-constrained control \citep{bokanowski2022optimistic, tassa2014control} offer complementary guarantees, but require online interaction, a dynamics model, or a verified backup controller, none of which is available in the offline setting.

\begin{wrapfigure}{r}{0.62\textwidth}
\vspace{-15pt}
\centering
\includegraphics[width=0.92\linewidth]{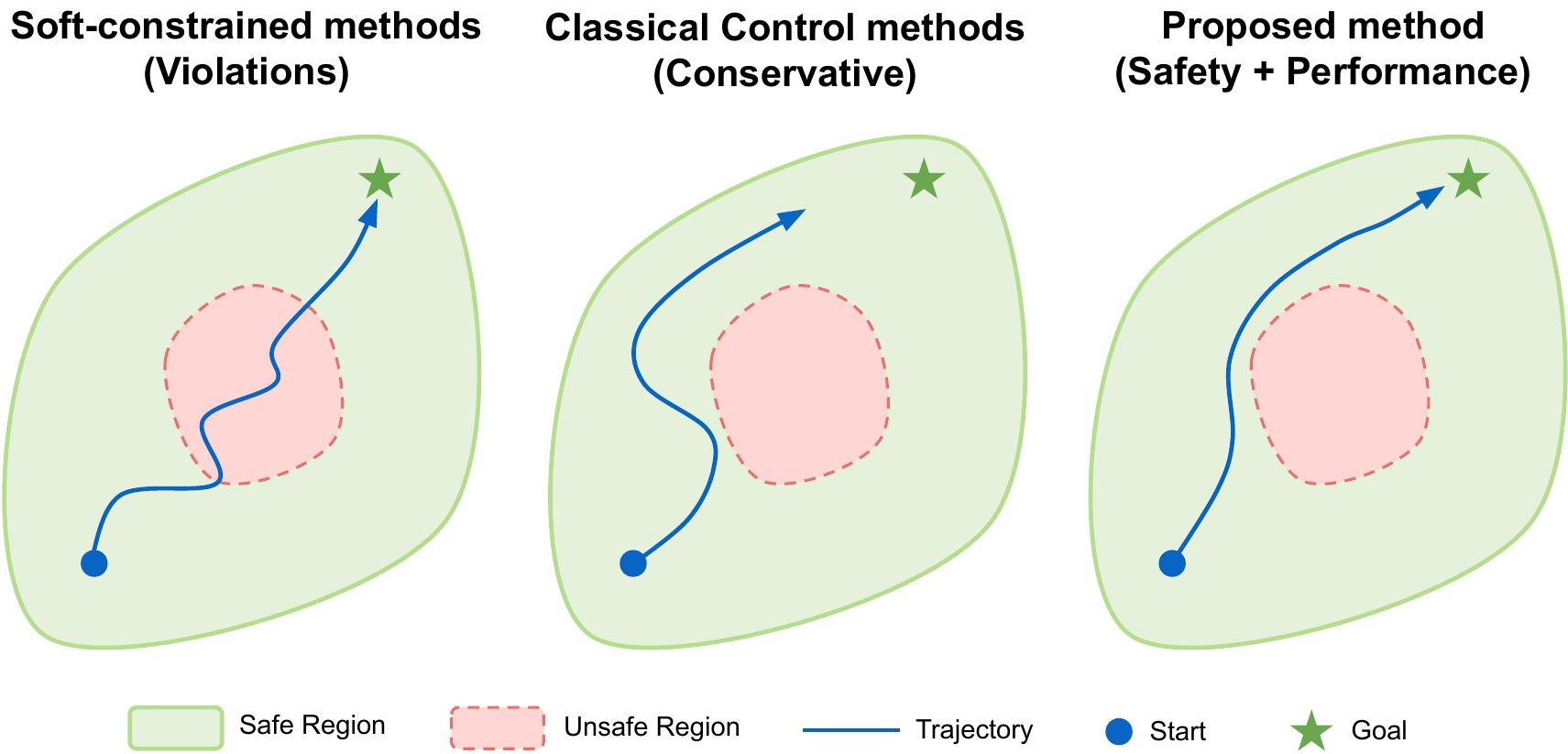}
\caption{\textbf{Safety--performance trade-off in offline RL (conceptual illustration):} soft-constraint methods violate constraints, while classical safety methods stay safe but follow conservative, lower-performing paths; our method achieves safety with high performance.}
\label{fig:prime}
\vspace{-15pt}
\end{wrapfigure}

A principled alternative is to formulate the problem as a \emph{state-constrained optimal control problem} (SC-OCP), where safety is enforced as a hard state constraint and performance is captured via a reward functional. 
Despite its strong theoretical foundations, solving SC-OCPs in practice is difficult: the existence and characterization of optimal solutions typically rely on stringent controllability conditions \citep{doi:10.1137/0324032}. To mitigate these difficulties, \cite{altarovici2013general} introduced an epigraph-based reformulation that characterizes the SC-OCP value function by computing its epigraph through dynamic programming, leading to a Hamilton--Jacobi--Bellman partial differential equation (HJB-PDE). The original value function and the associated optimal policy can then be recovered from this epigraph representation.
Nevertheless, solving HJB-PDEs is fundamentally limited by the curse of dimensionality, rendering them impractical for high-dimensional systems when using conventional numerical solvers \citep{mitchell2004toolbox, Hao2024csl}. The epigraph reformulation further augments the state space, compounding the computational burden and limiting applicability to low-dimensional settings. Although various approximation and acceleration strategies have been proposed, they often require restrictive assumptions on system structure or dynamics \citep{chow2017algorithm}. Consequently, computing solutions to the resulting HJB equations for general nonlinear, high-dimensional control systems remains an open and challenging problem.

In this work, we bridge the gap between rigorous state-constrained control and scalable offline RL by deriving a data-driven Bellman-style recursion for an \emph{epigraph feasibility value function}. This function captures the maximum achievable performance that remains compatible with future safety, enabling feasibility-aware value learning directly from offline data without solving high-dimensional HJB equations. To extract policies, we leverage Flow Matching as an efficient generative mechanism, guiding action synthesis toward regions that are both safe and high-performing within the support of the dataset. We call the resulting framework \emph{Epigraph-Guided Flow Matching} (EpiFlow).\footnote{
Project page: \url{https://tau-intelligence.com/EpiFlow/}.}

We evaluate the proposed approach across a range of safety-critical benchmarks, including navigation tasks, high-dimensional Safety Gymnasium environments \citep{ji2023safety}, and Bullet-Safety-Gym tasks \citep{Gronauer2022BulletSafetyGym}. Our results demonstrate competitive returns with near-zero empirical safety violations, outperforming both penalty-based and safety-filtering baselines.

\noindent\begin{minipage}{\linewidth}
The main contributions of this work are:
\begin{itemize}
\item We formulate safe offline RL as a state-constrained optimal control problem, providing a principled framework for jointly optimizing safety and performance.
\item We derive a bounded Bellman-style recursion for an epigraph value function, to solve for the resultant HJB-PDE, thereby enabling feasibility-aware learning directly from offline data.
\item We incorporate this formulation into a generative policy framework using Flow Matching, enabling efficient synthesis of safe and high-performing actions in continuous spaces.
\item We demonstrate strong empirical performance across safety-critical benchmarks, achieving near-zero violations while maintaining competitive returns.
\end{itemize}
\end{minipage}

\nocite{tayal2025physics, tayal2024learning}

%% file: sections/3_background.tex
\label{sec:problem-setup}

We study safe offline reinforcement learning in environments with hard state constraints. The environment is modelled as a state-constrained Markov Decision Process (SC-MDP), defined by the tuple
$\mathcal{M} = (\mathcal{X}, \mathcal{A}, f, r, \ell, \gamma)$,
where $\mathcal{X}$ and $\mathcal{A}$ denote the state and action spaces, $x_{t+1} = f(x_t, a_t)$ denotes an unknown, deterministic transition function, $r : \mathcal{X} \rightarrow \mathbb{R}$ is the reward function, $\ell : \mathcal{X} \rightarrow \mathbb{R}$ is an instantaneous state-based safety function, and $\gamma \in (0,1)$ is the discount factor.
We define the \emph{failure set} 
$\mathcal{F} := \{x \in \mathcal{X} \mid \ell(x) < 0\}$,
which represents unsafe states that must be avoided at all times (e.g., collisions or constraint violations). A trajectory is considered safe if it never enters $\mathcal{F}$.
We assume access to an offline dataset
$\mathcal{D} = \{(x_t, a_t, r_t, \ell_t, x_{t+1})\}$,
collected by an unknown behavior policy, with no further interaction with the environment permitted. A policy $\pi(a \mid x)$ induces trajectories
$\Gamma = (x_0, a_0, x_1, \dots)$
through the transition dynamics. Multiple constraints $\ell_i(x)\ge 0$ are handled without loss of generality via $\ell(x)=\min_i \ell_i(x)$, and the dynamics $f$ are never queried: our method learns only from $\mathcal{D}$.

\subsection{State-Constrained Optimal Control Problem}
\label{subsec: sc-ocp}
Given an initial state $x$, the objective is to compute the maximum achievable discounted return subject to strict state safety at all future time steps. This requirement can be formalized through the following \emph{state-constrained value function}:
\begin{equation}
\label{eq:scorl}
\begin{aligned}
    V(x)
\;=\;
\sup_{\pi}
\big[ \sum_{k=0}^{\infty} \gamma^k r(x_k) \,\big|\, x_0 = x \big],
\;\; \text{s.t.} \;\;
x_t \notin \mathcal{F}, \;\forall t \ge 0.
\end{aligned}
\end{equation}
Unlike formulations based on expected cumulative penalties, equation~\eqref{eq:scorl} encodes a \emph{hard safety requirement}, only policies that admit trajectories remaining entirely outside the failure set are considered feasible. This formulation directly captures safety-critical requirements where even a single violation is unacceptable.

However, directly optimizing equation~\eqref{eq:scorl} in the offline setting is challenging. Feasibility depends on long-horizon future behavior, which cannot be verified using local penalties or one-step constraints. Moreover, offline datasets often underrepresent unsafe transitions, making it difficult to infer whether a given state admits any safe continuation. As a result, standard offline RL objectives, which optimize expected return, are ill-suited for enforcing state-wise safety constraints.

\subsection{Epigraph Reformulation}
\label{subsec:epigraph}

The state-constrained value function in equation~\eqref{eq:scorl} jointly optimizes performance and safety, which makes it difficult to reason about directly in the offline setting. To expose the structure of this problem, we adopt an epigraph-based reformulation that expresses the original objective as a two-stage optimization.

In the first stage, the state-constrained value function is characterized as the largest performance threshold that is feasible under the safety constraint:
\begin{equation}
\label{eq:epigraph-stage1}
V(x)
=
\sup \{ z \in \mathbb{R} \mid \hat{V}(x,z) \ge 0 \}.
\end{equation}
This formulation makes explicit that $V(x)$ represents the maximum achievable return level that is compatible with safety starting from state $x$.

The second stage defines an auxiliary value function $\hat{V}(x,z)$, which evaluates whether a given performance threshold $z$ is feasible. Specifically, $\hat{V}(x,z)$ is defined as
\begin{equation}
\label{eq:epigraph-stage2}
\hat{V}(x_t, z_t)
=
\sup_{\pi}
\min \Big\{
\sum_{k=t}^{\infty} \gamma^{k-t} r(x_k) - z_t,
\;
\min_{s \ge t} \gamma^{s-t} \ell(x_s)
\Big\}.
\end{equation}

A formal derivation of the epigraph reformulation is provided in Appendix~\ref{appendix: aux_vfunc_proof}. Discounting is applied to the safety term solely to enable a Bellman-style discrete-time recursion (Equation~\eqref{eq:V_hat_recursion_final}); since $\gamma\in(0,1)$ preserves the sign of $\ell(x)$, this does not alter the hard state constraint and serves only as a technical device for dynamic programming (see Appendix~\ref{appendix: aux_vfunc_proof}). 
This epigraph feasibility value function $\hat{V}(x,z)$ is non-negative if and only if there exists a policy that both achieves discounted return at least $z$ and remains outside the failure set at all future time steps. As a result, equation~\eqref{eq:epigraph-stage1} selects the largest performance threshold for which the safety constraint is feasible.

The scalar variable $z$ therefore plays the role of a performance envelope: we call $z^*(x)=\sup\{z\mid\hat{V}(x,z)\ge 0\}$ the \emph{performance envelope} at $x$, and the zero-superlevel set $\{x\mid V_s(x)\ge 0\}$ of the safety value $V_s$ (Theorem~\ref{thm: recursion}) the \emph{safety envelope}. Values of $z$ smaller than the optimal threshold correspond to conservative but safe behavior, while values larger than this threshold are infeasible and necessarily lead to safety violations.

For discounted infinite-horizon problems, the auxiliary epigraph variable admits the discrete-time update
\begin{equation}
\label{eq:z-dynamics}
z_{t+1} = \frac{z_t - r(x_t)}{\gamma},
\end{equation}
which depends only on the current pair $(x_t,z_t)$, so the augmented state is Markov: $P(x_{t+1},z_{t+1}\mid x_t,z_t,a_t)=P(x_{t+1},z_{t+1}\mid x_{0:t},z_{0:t},a_t)$. For action-dependent rewards, the update becomes $z_{t+1}=(z_t-r(x_t,a_t))/\gamma$ and all results hold unchanged. This interpretation enables learning the feasibility structure of $V(x)$ from offline data using value-based methods, without directly enforcing constraints during policy optimization.

%% file: sections/4_methodology.tex
In Section~\ref{sec:problem-setup} we showed that the state-constrained value function $V(x)$ may be characterized via the feasibility epigraph value function $\hat{V}(x,z)$ through the two-stage formulation (Equations~\eqref{eq:epigraph-stage1},\eqref{eq:epigraph-stage2}). This section describes a practical, data-driven procedure to learn a dataset-consistent policy that is both safe and performant. We first derive the bounded recursive bellman equation for learning epigraph feasibility value function from offline data and then using it to synthesize a policy. Implementation-level details are summarized in Algorithm~\ref{alg: learn_epiflow} \& Algorithm~\ref{alg: policy_synthesis} and illustrated in Fig.~\ref{fig:framework}.
\vspace{-0.8em}

\subsection{Recursive Characterization} 
We begin by characterizing the structure of the feasibility epigraph value function. In the online setting, a Bellman-style recursion alone suffices for learning $\hat{V}$ \citep{so2023solving}. However, in the offline setting the recursion is ill-conditioned: because the Bellman update does not explicitly enforce dependence on the epigraph variable $z$, function approximators can become desensitized to $z$, causing $\hat{V}$ to collapse toward the safety value $V_s$ (see Appendix~\ref{appendix: eff_reg}). To address this, we state a \emph{bounded} recursion that couples the Bellman equation with a decomposition upper bound derived from the epigraph definition.

\begin{theorem}[Bounded Epigraph Value Recursion]\label{thm: recursion}
    Consider a deterministic SC-MDP with bounded reward and discount factor $\gamma \in (0,1)$. Given a state $x_t$, its successors $x_{t+1}=f(x_t,a_t)$ for $a_t\in\mathcal{A}$, and the epigraph update $z_{t+1}=(z_t-r(x_t))/\gamma$, define the optimal reward value $V_r(x_t) = \displaystyle\sup_{\pi}\sum_{k=t}^{\infty} \gamma^{k-t} r(x_k)$ and the optimal safety value $V_s(x_t) = \displaystyle\sup_{\pi} \displaystyle\min_{s \ge t} \gamma^{s-t} \ell(x_s)$. Then the epigraph value function $\hat{V}(x_t, z_t)$ satisfies:

    \textit{(i) Bellman recursion.}
    \begin{equation}
    \label{eq:V_hat_recursion_final}
    \hat{V}(x_t, z_t) \;=\; \min\Big\{ \ell(x_t),\; \gamma\, \max_{a_t \in \mathcal{A}}\, \hat{V}(x_{t+1}, z_{t+1}) \Big\}.
    \end{equation}

    \textit{(ii) Decomposition upper bound.}
    \begin{equation}
    \label{eq:aux_decomp}
    \hat{V}(x_t, z_t) \;\leq\; \min\big\{ V_r(x_t) - z_t,\; V_s(x_t) \big\}.
    \end{equation}
\end{theorem}

A formal proof is provided in Appendix~\ref{appendix: aux_vfunc_rec_proof}. Part~(i) follows from decomposing the discounted return and safety terms in equation~\eqref{eq:epigraph-stage2}, and convergence can be established using arguments similar to those for safety value functions in Hamilton--Jacobi reachability \citep{Fisac2019HJSafety, so2023solving}. Part~(ii) follows from the $\sup$-$\min$ inequality applied to the epigraph definition~\eqref{eq:epigraph-stage2}.


\paragraph{Why both parts are necessary.}
Part (i) alone is insufficient because the Bellman update for $\hat{V}(x,z)$ is structurally identical to that of a pure safety value function:
\begin{equation}\label{eq:collapse}
V_s(x_t) = \min\big\{\ell(x_t),\;\gamma\, \max_{a_t}\, V_s(x_{t+1})\big\}
    \;\;\longleftrightarrow\;\;
\hat{V}(x_t,z_t) = \min\big\{\ell(x_t),\;\gamma\, \max_{a_t}\, \hat{V}(x_{t+1},z_{t+1})\big\}.
\end{equation}
%
As a result, under limited offline data, a function approximator can minimize training error by learning a degenerate solution that ignores the epigraph variable and collapses $\hat{V}(x,z)$ to a state-only safety value $V_s(x)$. This collapse arises because $z$ influences the recursion only through a shifted argument, making its effect weak and difficult to identify from data (see Appendix~\ref{appendix: eff_reg}). 

However, this behavior is incorrect, in safe regions the value function should depend explicitly on the performance term $V_r(x) - z$. Part (ii) prevents this collapse by introducing an explicit linear dependence on $z$, forcing the model to retain sensitivity to the epigraph variable.
Together, the two parts ensure that $\hat{V}(x,z)$ captures both safety and performance, yielding a well-posed fixed-point suitable for offline learning (a detailed discussion is in Appendix~\ref{appendix: bounded_fp_discussion}).

\paragraph{Avoiding OOD actions.}
The maximization over actions in part~(i) (Theorem~\ref{thm: recursion}) can produce optimistic values for actions not in the dataset; in our setting this is a safety risk, since an overestimated out-of-distribution action may lead into $\mathcal{F}$. We therefore take a risk-averse approach and adopt expectile regression \citep{kostrikov2022offline}: an action-conditional $\hat{Q}(x,z,a)$ is fitted on dataset transitions and then distilled into $\hat{V}(x,z)$ via the asymmetric loss $\mathcal{L}^\tau(u)=|\tau-\mathbf{1}(u<0)|\,u^2$ with $\tau\in(0.5,1)$, which biases toward the upper envelope of dataset-supported values without querying unseen actions. The expectile distillation of $\hat{Q}$ into $\hat{V}$ thus plays the role of $\max_{a_t}$ in equation~\eqref{eq:V_hat_recursion_final}, restricted to dataset-supported actions.

\begin{figure*}
    \centering
    \includegraphics[width=\textwidth]{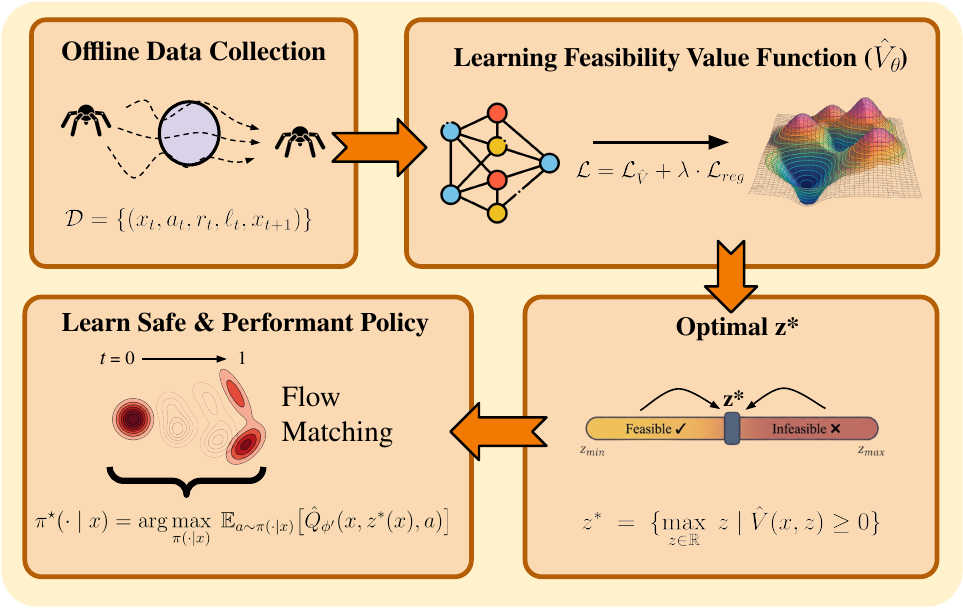}
    \vspace{-1.5em}
    \caption{\textbf{Framework Overview.} We propose a framework to solve the SC-OCP problem by using offline dataset. We learn $\hat{V}$, which if satisfies to stay positive, makes sure that the system stays safe while maximizing the reward. We then learn a policy which optimizes for $\hat{V}$ at all time steps to find the most optimal path that the offline data supports.}
    \label{fig:framework}
\vspace{-1.3em}
\end{figure*}

\subsection{Practical Loss Functions} \label{subsec: practical_loss}
Learning proceeds by alternating
fitting an action-value model, distilling expectile values, estimating the
separate reward and safety envelopes, and enforcing the decomposition
constraint. All expectations are empirical averages over minibatches drawn
from $\mathcal{D}$ together with $z$ values sampled from a dataset-supported
interval $[z_{\min},z_{\max}]$. The feasibility Q-target follows from
equation~\eqref{eq:V_hat_recursion_final},
\begin{equation}
\label{eq:Qhat_target}
y_Q(x,z,a,x',z') \;=\; \min\{ \ell(x),\; \gamma \hat{V}_{\theta}(x',z') \},
\end{equation}

the corresponding squared loss for which is
\begin{equation}
\label{eq:L_Qhat}
\mathcal{L}_{\hat{Q}}(\theta')=\mathbb{E}\big[ (y_Q(x,z,a,x',z') - \hat{Q}_{\theta'}(x,z,a))^2 \big].
\end{equation}

Expectile distillation to the state-value is performed by
\begin{equation}
\label{eq:L_Vhat}
\mathcal{L}_{\hat{V}}(\theta)=\mathbb{E}\big[ \mathcal{L}^\tau( \hat{Q}_{\theta'}(x,z,a) - \hat{V}_{\theta}(x,z) )\big],
\end{equation}

with $\tau$ chosen to bias toward the upper data-supported envelope.
To enforce the decomposition upper bound (Theorem~\ref{thm: recursion}, part~(ii)), we independently learn reward and safety value functions $V_r(x)$ and $V_s(x)$ using standard IQL-style losses (details in Appendix~\ref{appendix: reward_safety_losses}). The decomposition regularizer is then implemented as the one-sided penalty
\begin{equation}
\label{eq:L2_regulariser}
\begin{aligned}
\mathcal{L}_{reg}(\theta)
\;=\;
\mathbb{E}\Big[
\max\Big(
0,\;
\hat{V}_\theta(x,z) 
-
\min\big\{
V_r^\phi(x) - z,\;
V_s^\psi(x)
\big\}
\Big)
\Big],
\end{aligned}
\end{equation}

reduction of this loss would ensure satisfaction of equation \eqref{eq:aux_decomp}
and the full objective combines the terms:
\begin{equation}
\label{eq:total-loss}
\begin{aligned}
\mathcal{L}(\theta)
\;=\;&
\mathcal{L}_{\hat{V}}(\theta)
+
\lambda \cdot \mathcal{L}_{reg}(\theta)
\end{aligned}
\end{equation}

where the $\lambda$ controls the trade-offs between bellman regression and the decomposition regularizer.

The outcome of this learning procedure is a dataset-consistent feasibility value
function $\hat{V}(x,z)$ together with reward and safety envelopes $V_r(x)$ and
$V_s(x)$. 
The complete learning procedure is
summarized in Algorithm~\ref{alg: learn_epiflow}, and its interaction between the
different components is illustrated in Fig.~\ref{fig:framework}.

\subsection{Policy Synthesis} \label{ssec: policy_syn}

Given the learned epigraph feasibility value function $\hat{V}_\theta(x,z)$, we now describe how to synthesize a policy that improves performance while respecting the feasibility constraints. For each state $x$, feasibility induces a maximal admissible performance threshold
\begin{equation} \label{eq: z_opt_search}
z^*(x) = \sup \{ z \in \mathbb{R} \mid \hat{V}_\theta(x,z) \ge 0 \},
\end{equation}
which is obtained via binary search over $z$. Bisection is justified structurally: $\hat{V}(x,z)$ is non-increasing in $z$, since increasing $z$ lowers the first term in equation~\eqref{eq:epigraph-stage2} and leaves the second unchanged, so any non-monotonicity of $\hat{V}_\theta$ is approximation error; in practice, bisection over $[z_{\min},z_{\max}]$ was stable across all tasks and seeds. If no threshold is feasible, $z^*(x)$ clamps to $z_{\min}$ and the policy acts to recover safety rather than declining to act.

Policy improvement seeks actions with high feasible value while staying near the offline data distribution. While recent offline RL methods often use weighted regression with rejection sampling \citep{liu2025ciql, wang2023diffusionql, zheng2024fisor}, such sampling increases inference cost. We therefore follow an approach similar to Flow Q-Learning~\citep{park2025fql} and use a dataset behavior policy ($\pi_\beta$) based flow-matching model as a \emph{flow behavior teacher} ($\mu_\zeta$), which is distilled into a one-step actor ($\mu_\omega$) that directly optimizes

\begin{equation}
\label{eq:policy-obj}
\pi^\star(\cdot\mid x)
=
\arg\max_{\pi(\cdot\mid x)}
\begin{aligned}
\;\mathbb{E}_{a\sim\pi(\cdot\mid x)}
\big[\hat{Q}_{\theta'}(x,z^*(x),a)\big] \;\;
 \text{s.t.} \;\; Div\big(\pi(\cdot\mid x)\;||\;\pi_\beta(\cdot\mid x)\big) \le \epsilon
\end{aligned}
\end{equation}

where $\mu_\omega$ is the desired $\pi^\star$, $Div$ depicts divergence, $\pi_\beta(a\mid x)$ denote the behavior policy and $\epsilon$ is a small positive constant.

\paragraph{Learning the Policy.}
To optimize for $\hat{Q}_{\theta'}(x,z^*(x),a)$ while staying close to the offline data distribution, a \emph{flow behavior teacher} $\mu_\zeta(x,\eta)$ is first trained by conditional flow matching, where $x$ is the state and $\eta \sim \mathcal{N}(0,I)$ is a latent noise variable. The teacher is parameterized by a time-dependent velocity field $v_\zeta(x,\eta,t)$ and is trained with the regression objective
\begin{align}
\mathcal L_{\text{flow}}(\zeta)
=
\mathbb E_{(x,a)\sim\mathcal D,\,\eta\sim\mathcal N(0,I),\,t\sim\mathcal U([0,1])}
\!\left[
\left\|v_\zeta(x,\eta_t,t)-(a-\eta)\right\|_2^2
\right],
\end{align}
where $\eta_t=(1-t)\eta+ta$. After training, integrating $v_\zeta$ from $t=0$ to $t=1$ maps $\eta$ to an action sampled from the data-supported behavior policy. We then train a one-step student actor $\mu_\omega(x,\eta)$ that directly predicts an action from $(x,\eta)$ without iterative integration. The student is distilled from the teacher through
\begin{align}
\mathcal L_{\text{distill}}(\omega)
=
\mathbb E_{(x,\eta)\sim\mathcal D\times\mathcal N(0,I)}
\!\left[
\left\|\mu_\omega(x,\eta)-\mu_\zeta(x,\eta)\right\|_2^2
\right],
\end{align}
which regularizes the actor toward the support of the offline dataset and reduces extrapolation error. Finally, we train the deployment policy with
\begin{equation}
\mathcal{L}_{\text{actor}}(\omega)
=
\mathbb{E}_{x\sim\mathcal{D},\,\eta\sim\mathcal{N}(0,I)}
\big[
-\hat{Q}_{\theta'}(x,z^*(x),\mu_\omega(x,\eta))
\big]
+
\alpha \mathcal{L}_{\text{distill}}(\omega),
\end{equation}
where $\alpha$ balances Q-maximization and behavior regularization. Unlike FQL, the actor maximizes the feasibility value $\hat{Q}_{\theta'}$ at the state-wise budget $z^*(x)$ rather than an unconstrained reward Q-function, and since $\hat{Q}$ is fitted in-sample via expectile regression, conservatism is enforced at the value level rather than by the distillation term alone. Since $\mu_\omega$ is a one-step policy, it is the only component queried at deployment, giving a single forward pass per action.

%% file: sections/6_experiments.tex
\begin{figure*}
    \centering  
    \includegraphics[width=0.99\textwidth]{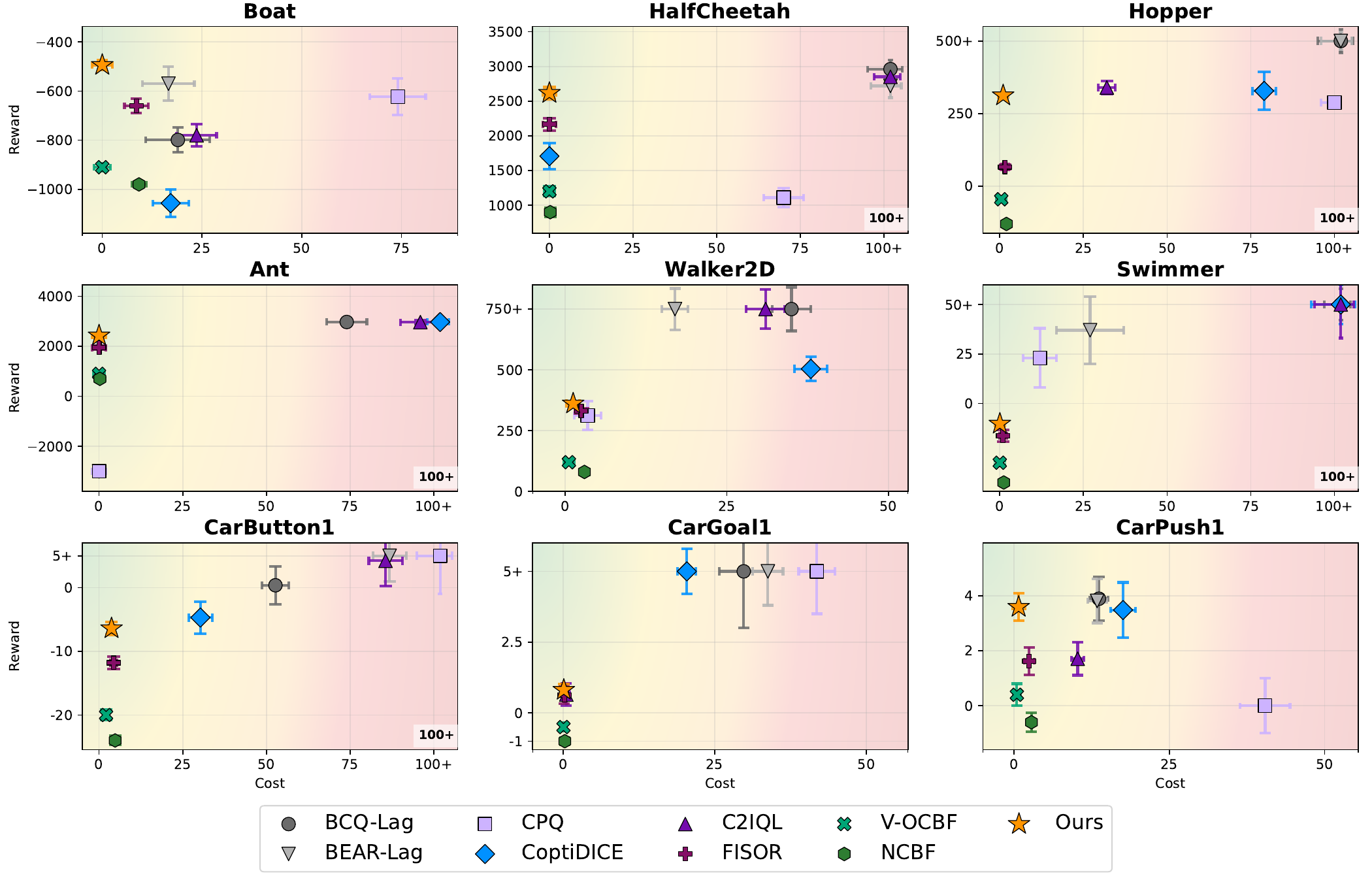}
    \vspace{-0.4em}
    \caption{\textbf{Evaluation Results:} Results plot for all the environments. Points towards \textcolor{teal}{left} ($\leftarrow$) are more safe than those on \textcolor{red}{right} ($\rightarrow$). Whereas those towards \textcolor{teal}{top} ($\uparrow$) have higher rewards than those towards \textcolor{red}{bottom} ($\downarrow$). Evaluated over 500 episodes and 5 seed values  ($\pm$ standard deviation).}
    \label{fig:results}
\vspace{-1.0em}
\end{figure*}

\textbf{Evaluation Metrics:} We would like to answer two questions: (i) how our method's safety rate compares against state-of-the-art safe offline RL baselines, and (ii) how safety compliance affects task performance, measured by the cumulative reward $R(x)=\sum_{k=0}^\infty \gamma^k r(x_k)$. Accordingly, we evaluate methods along two axes: \emph{safety/cost}, measured by the total number of safety violations before episode termination, and \emph{performance}, measured by cumulative episode reward. This allows us to assess the safety--performance trade-off across different offline RL approaches. For tabulated results, we follow the DSRL protocol \citep{liu2024dsrl} and report normalized reward $(R_\pi-r_{\min})/(r_{\max}-r_{\min})$ and normalized cost $(C_\pi+\epsilon)/(\kappa+\epsilon)$ with cost limit $\kappa=10$, where a normalized cost below 1 is deemed safe; $\kappa$ is used only for reporting, as our method does not require a cost limit.

\textbf{Baselines:} We evaluate against two classes of safe offline RL methods that reflect different trade-offs between performance and safety. The first class focuses on maximizing return while incorporating safety constraints, typically through Lagrangian or regularization-based formulations. This includes \textbf{BEAR-Lag} (a Lagrangian variant of \citep{kumar2019stabilizing}), \textbf{BCQ-Lag} (a Lagrangian extension of \citep{fujimoto2019off}), as well as constraint-aware offline RL methods such as \textbf{COptiDICE} \citep{lee2022coptidice}, \textbf{CPQ} \citep{xu2022constraints}, and \textbf{C2IQL} \citep{liu2025ciql}. 
The second class emphasizes strict safety, often at the expense of reward maximization, and includes methods that prioritize absolute safety. This category comprises of offline HJ Reachability and Control Barrier Function (CBF)-based approaches such as \textbf{FISOR} \citep{zheng2024fisor},  \textbf{NCBF} \citep{Alexander2020NCBF} and \textbf{V-OCBF} \citep{tayal2026vocbf}.

\vspace{-0.5em}
\subsection{Experimental Case Studies}
\vspace{-0.5em}
For experiments, we test and evaluate our framework against the baselines on a varied range of environments. Below we list all these environments that we use:
\vspace{-0.5em}
\begin{itemize}
    \item \textbf{Safe Boat Navigation:} In our first experiment, we examine a $2$-dimensional collision avoidance problem involving a boat in a river with drift velocity that varies with the y-coordinate of the boat. The objective is to ensure safety by avoiding the static obstacles while navigating through the river with varying drift. Further details on the system dynamics, state space bounds, and experimental setup are provided in Appendix~\ref{appendix: Boat2D}.
    
    \item \textbf{Safety Gymnasium:} We also evaluate our framework on Safety Gymnasium environments \citep{ji2023safety} which include Safe Velocity and Safe Navigation Gymnasium environments. Specifically, we use high-dimensional MuJoCo tasks like Hopper, Half Cheetah, Swimmer, Walker2D and Ant for Safe Velocity task, and Car Navigation task with 3 navigation environment types: Goal, Push, Button.

    \item \textbf{Bullet-Safety-Gym:} We further evaluate on the eight Run and Circle tasks with Ant, Ball, Car, and Drone agents \citep{Gronauer2022BulletSafetyGym} from the DSRL benchmark (Table~\ref{tab:bullet_results}, Appendix~\ref{appendix: env_desc}).
\end{itemize}
\vspace{-0.5em}

For all standard Safety Gymnasium environments, we retain the original reward and safety-violation metrics from Safety-Gymnasium \citep{ji2023safety} and use standard DSRL dataset for safe offline RL \citep{liu2024dsrl}. To evaluate our method against baselines on all the environments, we fix 500 randomly sampled safe initial states for each environment, shared across all methods, and perform rollouts over 5 different training seeds, which control network initialization and minibatch sampling.

\subsection{Results}

We evaluate our framework against a comprehensive set of offline RL baselines across all environments described above. Figure~\ref{fig:results} summarizes the results (complete numerical results in Table~\ref{tab:main_results}, Appendix~\ref{appendix: env_desc}), illustrating how different methods navigate the trade-off between cumulative reward and safety constraint satisfaction. Existing safe offline RL approaches vary widely in their ability to balance these objectives. Lagrangian-based methods such as \textbf{BEAR-Lag} and \textbf{BCQ-Lag} frequently favor reward maximization, resulting in elevated safety violations due to weak constraint enforcement. Other approaches, including \textbf{COptiDICE}, \textbf{CPQ}, and \textbf{C2IQL}, exhibit inconsistent behavior across tasks, alternating between overly conservative policies that sacrifice performance and policies that fail to maintain reliable safety rates. In contrast, strong safety-constrained methods such as \textbf{FISOR}, \textbf{NCBF}, and \textbf{V-OCBF} achieve strong safety performances across environments, but typically doing so by adopting conservative policies that limit achievable reward.

On the other hand, our framework consistently attains high safety rates while maintaining strong performance across all evaluated domains, ranging from low-dimensional Boat Navigation to high-dimensional MuJoCo Safe Velocity and Safe Car Navigation tasks. Unlike prior methods that trade reward for safety or vice versa, our framework reliably operates in the desirable regime where both objectives are jointly optimized. By leveraging the epigraph feasibility value function, our approach anticipates future safety violations and allocates performance budgets accordingly, avoiding both excessive conservatism and unsafe reward-seeking. The small residual costs on tasks such as Hopper, CarButton1, and CarPush1, which the safety-filtering baselines also incur, stem from sparse dataset coverage near the constraint boundary, where approximation errors in $\hat{V}_\theta$ can yield a slightly optimistic budget $z^*(x)$. On the Bullet-Safety-Gym tasks (Table~\ref{tab:bullet_results}), our method attains the highest return among methods with zero observed violations on all the tasks. These results demonstrate that our framework offers a more stable and effective solution for safety-critical offline reinforcement learning than existing constraint-based or filtering-based baselines.

\begin{wrapfigure}{r}{0.43\textwidth}
\vspace{-15pt}
\centering
\includegraphics[width=\linewidth]{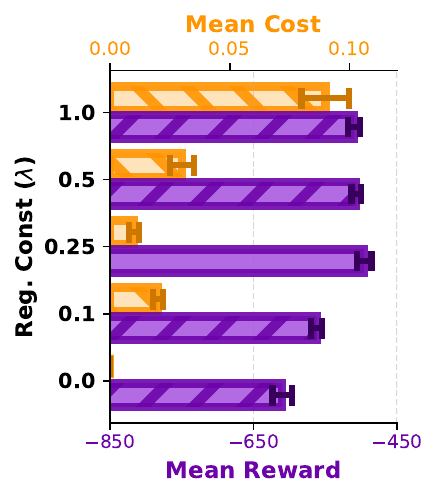}
\caption{Effect of Regularizer parameter ($\lambda$) ($\pm$ standard deviation).}
\label{fig: lambda_analysis}
\vspace{-25pt}
\end{wrapfigure}

\subsection{Ablation Studies}
\label{subsec: abl_std}

To better understand the design choices and robustness, we conduct a series of ablation studies that isolate the effect of key modeling and inference components.

\paragraph{Effect of Regularizer Parameter ($\lambda$).}
We study how the decomposition regularizer (term $\mathcal{L}_{reg}$ in equation~\eqref{eq:total-loss}) affects the learned feasibility epigraph value
$\hat{V}_{\theta}$ by sweeping $\lambda\in\{0,0.1,0.25,0.5,1.0\}$. Figure~\ref{fig: lambda_analysis} summarizes downstream metrics (mean reward, mean cost).


The results in Figure \ref{fig: lambda_analysis} indicate that across different values of $\lambda$, a value which balances safety and reward somewhere in the middle ground should be optimal, the learned value function shouldn't drift away from the primary recursion formulation objective while also strictly complying with the upper bound. Based on these results, we use $\lambda=0.25$ as the preferred regularizer constant for our experiments. For more detailed explanation of this experiment, refer Appendix \ref{appendix:reg_ablation}. 

\paragraph{Sensitivity to Expectile Parameter ($\tau$).}
To analyze the sensitivity of the expectile regression parameter used to learn state value functions in Section~\ref{section: method}, we perform an ablation study to understand how this expectile parameter $\tau$ would affect the reward performance and safety in terms of cost. From the Figure~\ref{fig: expectile_plot}, we witness that a higher $\tau$ yields higher safety while keeping rewards competitive.



\begin{wrapfigure}{r}{0.43\textwidth}
\vspace{-10pt}
\centering
\includegraphics[width=\linewidth]{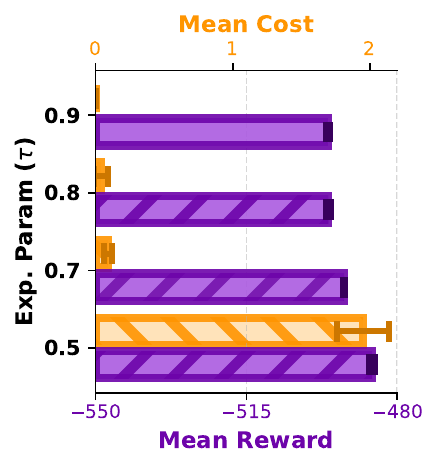}
\caption{Sensitivity to Expectile Parameter ($\tau$) ($\pm$ standard deviation).}
\label{fig: expectile_plot}
\vspace{-35pt}
\end{wrapfigure}

As our original objective from Section~\ref{subsec: sc-ocp} is to comply with the constraint while maximizing the reward, it is essential that we pick a $\tau$ so that it achieves highest safety rates. Therefore, based on this ablation study, we prefer $\tau=0.9$ expectile parameter for all our experiments. A more detailed study on this ablation can be referred to in Appendix~\ref{appendix: tau-sensitivity}. 

\paragraph{Computation Time Analysis}
The plots in Fig.~\ref{fig: computation_time_analysis} illustrate a variance in computational demands among existing methods when studied over 5 trials, some baselines require extensive training times while others face significant inference time latency. In contrast, our proposed framework naturally maintains a balance, staying on the lower side of both metrics. Overall, this demonstrates that our method does not require extensive compute resources and can be implemented efficiently within a highly reasonable timeframe.

\paragraph{Additional Ablation Studies.}
We further perform some more ablations to examine (i) the robustness of the learned policy to action perturbations, which probes the stability of epigraph-guided synthesis under execution noise; (ii) perform a comparison with Diffusion Model based counterpart in place of Flow Matching for our framework to show that our framework is agnostic of any generative model; and (iii) compare against FISOR-FM, which pairs FISOR's safety value with our flow-matching extractor, to isolate the contribution of the epigraph formulation, together with a sensitivity analysis of the distillation coefficient $\alpha$ (Tables~\ref{tab:fisor_fm} and \ref{tab:alpha_sweep}).
Together, these studies provide insight into the stability, hyperparameter sensitivity, and practical deployment characteristics of our framework. Detailed quantitative results and visualizations are reported in Appendix~\ref{appendix: ablation_study}.

\begin{figure}[ht]
    \centering
    \includegraphics[width=\linewidth]{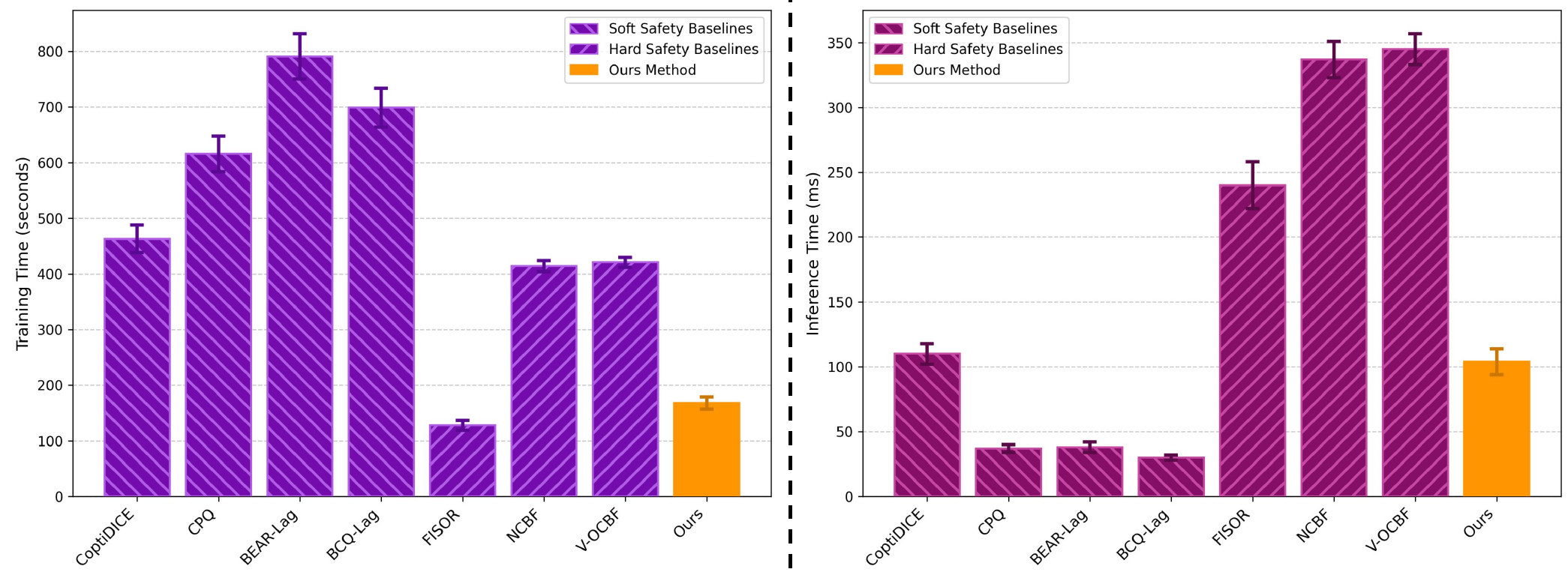}
    \caption{\textbf{Computation Time Analysis.} \emph{(Left)} shows training time comparison while \emph{(Right)} shows inference time ($\pm$ standard deviations) comparison of Our framework against all baselines.}
    \label{fig: computation_time_analysis}
\end{figure}

%% file: sections/7_conclusion.tex

In this work, we proposed an offline reinforcement learning framework that co-optimizes safety and performance by casting state-constrained control as a tractable epigraph-based value synthesis problem. The method learns an epigraph feasibility value function via bounded bellman equation and synthesizes executable continuous-action policies using flow matching, while remaining grounded in the offline dataset. Empirically, the approach achieves consistently low safety violations across high-dimensional control tasks while maintaining competitive performance.

While our framework reliably learns safe and performant policy from offline demonstrations, it is not formally verified, function approximation errors mean guarantees remain empirical. Furthermore, the framework assumes deterministic dynamics and is not explicitly trained to handle disturbances: with stochastic transitions, the worst-case successor required by the robust epigraph recursion \citep{altarovici2013general} cannot be identified from an offline dataset that records a single successor per transition. To close this gap, future work will explore integration of formal-certification (e.g., Lipschitz-based verification \citep{tayal2024learning} and conformal prediction \citep{lindemann2025formal}). Moreover, we will explore robust extensions of the framework that incorporate adversarial or uncertainty-aware training objectives, including risk measures such as CVaR that account for the severity, not only the occurrence, of violations \citep{rockafellar2007coherent}.

%% file: sections/Appendix.tex
\section{Theoretical Insights}

\subsection{Derivation of the Epigraph Reformulation}
\label{appendix: aux_vfunc_proof}

In this section, we formally derive the feasibility epigraph value function and prove its equivalence to the original state-constrained objective. We show that maximizing the constrained value is equivalent to finding the maximal level-set of an unconstrained feasibility value function defined over an augmented state space.

\subsubsection{Problem Formulation}
Consider a deterministic Markov Decision Process (MDP) with state dynamics $x_{k+1} = f(x_k, a_k)$ and a reward function $r(x_k)$. We enforce a state constraint $\ell(x_k) \ge 0$ for all time steps $k \ge t$.

Let $\Pi$ denote the set of all admissible policies. The state-constrained value function $V(x_t)$ is defined as the supremum of the cumulative discounted return subject to safety constraints:
\begin{equation}
    \label{eq:original_primal}
    V(x_t) := \sup_{\pi \in \Pi} \left\{ \sum_{k=t}^{\infty} \gamma^{k-t} r(x_k) \;\middle|\; \ell(x_k) \ge 0, \forall k \ge t \right\}.
\end{equation}

\subsubsection{Epigraph Value Function}
We introduce an auxiliary scalar decision variable $z_t \in \mathbb{R}$ (or epigraph variable), which represents a candidate performance threshold \citep{boyd2004convex}. We define the \textit{feasibility epigraph value function} $\hat{V}(x_t, z_t)$ as the optimal value of a maximin objective that jointly considers the return margin and the safety margin:
\begin{equation}
    \label{eq:epigraph_def}
    \hat{V}(x_t, z_t) := \sup_{\pi \in \Pi} \min \left\{ \sum_{k=t}^{\infty} \gamma^{k-t} r(x_k) - z_t, \quad \inf_{s \ge t} \gamma^{s-t} \ell(x_s) \right\}.
\end{equation}

\begin{proposition}[Equivalence of Epigraph Formulation]
    The original constrained value function $V(x_t)$ can be recovered from the zero-level set of the epigraph feasibility value function $\hat{V}(x_t, z_t)$. Specifically:
    \begin{equation}
        V(x_t) = \sup \{ z_t \in \mathbb{R} \mid \hat{V}(x_t, z_t) \ge 0 \}.
    \end{equation}
\end{proposition}

\begin{proof}
    We proceed by transforming the optimization problem in \eqref{eq:original_primal} using the epigraph form. Let $J(\pi; x_t) = \sum_{k=t}^{\infty} \gamma^{k-t} r(x_k)$ denote the return and $S(\pi; x_t) = \inf_{s \ge t} \ell(x_s)$ denote the safety margin of a policy. The constraint $\ell(x_k) \ge 0 \; \forall k$ is equivalent to $S(\pi; x_t) \ge 0$.
    
    The original problem can be rewritten by introducing the variable $z_t$ to represent the objective value:
    \begin{equation}
        \begin{aligned}
            V(x_t) = \sup_{\pi, z_t} \quad & z_t \\
            \text{s.t.} \quad & J(\pi; x_t) \ge z_t, \\
                              & S(\pi; x_t) \ge 0.
        \end{aligned}
    \end{equation}
    The two inequality constraints must hold simultaneously. This is logically equivalent to requiring the minimum of the constraint residuals to be non-negative. Note that $J(\pi; x_t) \ge z_t \iff J(\pi; x_t) - z_t \ge 0$. Thus, the constraints are satisfied if and only if:
    \begin{equation}
        \min \{ J(\pi; x_t) - z_t, \; S(\pi; x_t) \} \ge 0.
    \end{equation}
    Substituting this back into the optimization problem:
    \begin{equation}
        \begin{aligned}
            V(x_t) &= \sup_{z_t} \left\{ z_t \;\middle|\; \exists \pi \in \Pi \text{ s.t. } \min \{ J(\pi; x_t) - z_t, S(\pi; x_t) \} \ge 0 \right\} \\
                   &= \sup_{z_t} \left\{ z_t \;\middle|\; \sup_{\pi \in \Pi} \min \{ J(\pi; x_t) - z_t, S(\pi; x_t) \} \ge 0 \right\}.
        \end{aligned}
    \end{equation}
    To align this with the definition in \eqref{eq:epigraph_def}, we observe that the sign of the safety term is invariant under multiplication by a positive discount factor $\gamma^{s-t} > 0$. Therefore, $S(\pi; x_t) \ge 0$ is equivalent to $\inf_{s \ge t} \gamma^{s-t} \ell(x_s) \ge 0$.
    
    Replacing the safety term, the inner optimization becomes exactly $\hat{V}(x_t, z_t)$. Thus:
    \begin{equation}
        V(x_t) = \sup \{ z_t \in \mathbb{R} \mid \hat{V}(x_t, z_t) \ge 0 \}.
    \end{equation}
\end{proof}

\paragraph{Remark on Discounted Safety.}
In \eqref{eq:epigraph_def}, the safety constraint $\ell(x_s)$ is scaled by the discount factor $\gamma^{s-t}$. Since $\gamma \in (0, 1)$, this scaling does not alter the sign of the term, nor does it change the feasibility boundary (i.e., $\gamma^{s-t}\ell(x_s) \ge 0 \iff \ell(x_s) \ge 0$). This modification is introduced solely to ensure the resulting operator is a contraction, enabling the use of standard discrete-time Bellman recursions for learning $\hat{V}$.

\subsection{Theorem (\ref{thm: recursion})}
\label{appendix: aux_vfunc_rec_proof}

\begin{tcolorbox}[colback=gray!10, colframe=gray!80, boxrule=0.5pt, arc=4pt]
\textbf{Theorem \ref{thm: recursion}} (Bounded Epigraph Value Recursion) Consider a deterministic SC-MDP with bounded reward and discount factor $\gamma \in (0,1)$. Given a state $x_t$, its successors $x_{t+1}=f(x_t,a_t)$ for $a_t\in\mathcal{A}$, and the epigraph update $z_{t+1}=(z_t-r(x_t))/\gamma$, define $V_r(x_t) = \sup_{\pi}\sum_{k=t}^{\infty} \gamma^{k-t} r(x_k)$ and $V_s(x_t) = \sup_{\pi} \min_{s \ge t} \gamma^{s-t} \ell(x_s)$. Then the epigraph value function $\hat{V}(x_t, z_t)$ satisfies:

\textit{(i) Bellman recursion.}
\begin{equation*}
\hat{V}(x_t,z_t)
=
\min\big\{
\ell(x_t),
\;
\gamma\, \max_{a_t \in \mathcal{A}}\, \hat{V}(x_{t+1},z_{t+1})
\big\}.
\end{equation*}

\textit{(ii) Decomposition upper bound.}
\begin{equation*}
\hat{V}(x_t, z_t) \;\leq\; \min\big\{ V_r(x_t) - z_t,\; V_s(x_t) \big\}.
\end{equation*}
\end{tcolorbox}

\begin{proof}
We prove each part separately.

\textbf{Part~(i): Bellman recursion.}
We begin from the definition of the epigraph feasibility value function in \eqref{eq:epigraph-stage2}. For any policy $\pi$ and initial state
$x_t$, decompose the two terms inside the minimum.

First, the discounted return term admits the decomposition
\begin{equation}
\sum_{k=t}^{\infty}\gamma^{k-t} r(x_k) - z_t
=
\gamma\Big(
\sum_{k=t+1}^{\infty}\gamma^{k-(t+1)} r(x_k) - z_{t+1}
\Big),
\end{equation}
where the relation follows from the definition
$z_{t+1}=(z_t-r(x_t))/\gamma$.

Second, the discounted safety term satisfies
\begin{equation}
\min_{s\ge t}\gamma^{s-t}\ell(x_s)
=
\min\big\{
\ell(x_t),
\;
\gamma \min_{s\ge t+1}\gamma^{s-(t+1)}\ell(x_s)
\big\}.
\end{equation}

Substituting these decompositions into
\eqref{eq:epigraph-stage2} yields
\begin{equation}
\hat{V}(x_t,z_t)
=
\sup_{\pi}
\min\Big\{
\gamma
\big(
\sum_{k=t+1}^{\infty}\gamma^{k-(t+1)} r(x_k) - z_{t+1}
\big),
\;
\ell(x_t),
\;
\gamma \min_{s\ge t+1}\gamma^{s-(t+1)}\ell(x_s)
\Big\}.
\end{equation}

Since $\ell(x_t)$ does not depend on the actions and $\gamma>0$, the minimum can be reordered, and the supremum over policies separates into a maximization over the current action $a_t$, which determines $x_{t+1}=f(x_t,a_t)$, and a supremum over the policy from time $t+1$ onward:
\begin{equation}
\hat{V}(x_t,z_t)
=
\min\Big\{
\ell(x_t),
\;
\gamma\, \max_{a_t \in \mathcal{A}}\, \sup_{\pi}
\min\Big(
\sum_{k=t+1}^{\infty}\gamma^{k-(t+1)} r(x_k) - z_{t+1},
\;
\min_{s\ge t+1}\gamma^{s-(t+1)}\ell(x_s)
\Big)
\Big\}.
\end{equation}

Recognizing the inner supremum as the definition of
$\hat{V}(x_{t+1},z_{t+1})$ completes the proof of part~(i) and yields the recursion
\eqref{eq:V_hat_recursion_final}: $\hat{V}(x_t,z_t)=\min\big\{\ell(x_t),\;\gamma\, \max_{a_t \in \mathcal{A}}\, \hat{V}(x_{t+1},z_{t+1})\big\}.$

\textit{Remark.} An analogous Bellman recursion for the epigraph value function in the online (on-policy) setting was independently established by \citet{so2023solving}; part~(i) extends their result to the offline, dataset-driven regime.

\textbf{Part~(ii): Decomposition upper bound.}
Starting from the definition \eqref{eq:epigraph-stage2}, we apply the standard $\sup$-$\min$ inequality ($\sup \min\{A,B\} \leq \min\{\sup A, \sup B\}$) to separate the supremum over policies from the minimum over the two terms:
\begin{equation}
\begin{aligned}
\hat{V}(x_t, z_t)
&= \sup_{\pi} \min \Big\{
\sum_{k=t}^{\infty} \gamma^{k-t} r(x_k) - z_t,\;
\min_{s \ge t} \gamma^{s-t} \ell(x_s)
\Big\} \\
&\leq \min \Big\{
\sup_{\pi}\sum_{k=t}^{\infty} \gamma^{k-t} r(x_k) - z_t,\;
\sup_{\pi}\min_{s \ge t} \gamma^{s-t} \ell(x_s)
\Big\}
= \min\{ V_r(x_t) - z_t,\; V_s(x_t) \}.
\end{aligned}
\end{equation}
The inequality is strict whenever the policies that individually maximize the return and the safety margin differ, which is generically the case.
\end{proof}

\subsection{Bounded Fixed-Point Characterization: Detailed Discussion}
\label{appendix: bounded_fp_discussion}

This section expands on the paragraph \emph{``Why both parts are necessary''} in Section~\ref{section: method}.

\paragraph{Role of part~(i): Bellman recursion.}
The recursion in Theorem~\ref{thm: recursion}(i) characterizes the fixed-point structure of $\hat{V}$ and drives the value estimate upward toward the tightest feasible boundary supported by the data dynamics. Without this recursion, any function satisfying the upper bound in part~(ii) would be consistent --- including the trivial solution $\hat{V}(x,z)\to-\infty$, which satisfies the inequality vacuously but is uninformative.

\paragraph{Role of part~(ii): decomposition upper bound.}
Although the recursion characterizes the correct fixed point, it is ill-conditioned for offline learning. Observe that the Bellman target in part~(i) is structurally identical to the safety value recursion (equation~\eqref{eq:collapse} in the main text): $z$ enters only through $z_{t+1}$, and the $\min\{\ell,\gamma(\cdot)\}$ operator is the same as for $V_s$. In practice, neural networks exploit this structural degeneracy and converge to solutions where $\hat{V}(x,z)\approx V_s(x)$ for all $z$, as demonstrated by the contour plots in Appendix~\ref{appendix: eff_reg}.

When this collapse occurs in safe regions where $V_s(x)$ is large, the correct value should be $\hat{V}(x,z)\approx V_r(x)-z$, which is strictly smaller than $V_s(x)$. A collapsed estimate ignoring $z$ therefore violates the decomposition bound. By enforcing part~(ii) as a regularizer, we inject an explicit linear dependence on $z$ through $V_r(x)-z$, preventing the network from disregarding the epigraph variable.

\paragraph{Complementary roles.}
In summary, part~(i) identifies the correct fixed point from \emph{below} (it pushes $\hat{V}$ upward), while part~(ii) constrains it from \emph{above} (it prevents collapse). Together they define a bounded fixed-point characterization that is well-conditioned for offline learning with function approximation.

\subsection{Effect of the Decomposition Regularizer on Epigraph Sensitivity}
\label{appendix: eff_reg}

This appendix analyzes the effect of the decomposition-based regularizer (equation~\eqref{eq:L2_regulariser}) on the learned feasibility epigraph value function $\hat{V}(x,z)$. Figure~\ref{fig:unreg_vs_reg_vhat} shows contour plots of $\hat{V}$ evaluated at multiple values of $z$, comparing the unregularized objective with the regularized variant using $\lambda=0.25$. In the absence of regularization, the learned $\hat{V}_\theta$ exhibits nearly identical contour geometry across different $z$, indicating weak sensitivity to the epigraph variable. This behavior is consistent with the analysis in Section \ref{section: method}, i.e., although $z_t$ appears in the recursion, the Bellman update does not explicitly enforce conditioning on $z$, allowing the network to collapse to a representation that largely ignores the epigraph dimension.

Introducing the decomposition regularizer restores meaningful dependence on $z$. The regularizer enforces the inequality $\hat{V}(x,z)\le \min\{V_r(x)-z,\,V_s(x)\}$, which injects an explicit linear
dependence on $z$ through the $V_r(x)-z$ term. As a result, the learned $\hat{V}_\theta$ must adapt as $z$ varies in order to remain feasible, leading to visibly distinct level sets across different $z$ values in the regularized case. This effect prevents uncontrolled drift along the epigraph dimension and ensures that $\hat{V}_\theta$ encodes a nontrivial performance envelope.

Crucially, restoring sensitivity to $z$ is necessary for reliable extraction of the state-wise threshold $z^*(x)=\sup\{z\mid\hat{V}(x,z)\ge 0\}$. When $\hat{V}_\theta$ is nearly invariant to $z$, the reward threshold $z$ becomes ill-defined and unstable, undermining policy synthesis. The regularizer therefore plays a structural role as ensures that the learned feasibility value function preserves the intended semantics of the epigraph formulation. Additionally, an ablation study on different values of $\lambda$ is provided in Section~\ref{appendix:reg_ablation}.

\begin{figure}[ht]
    \centering
    \includegraphics[width=\linewidth]{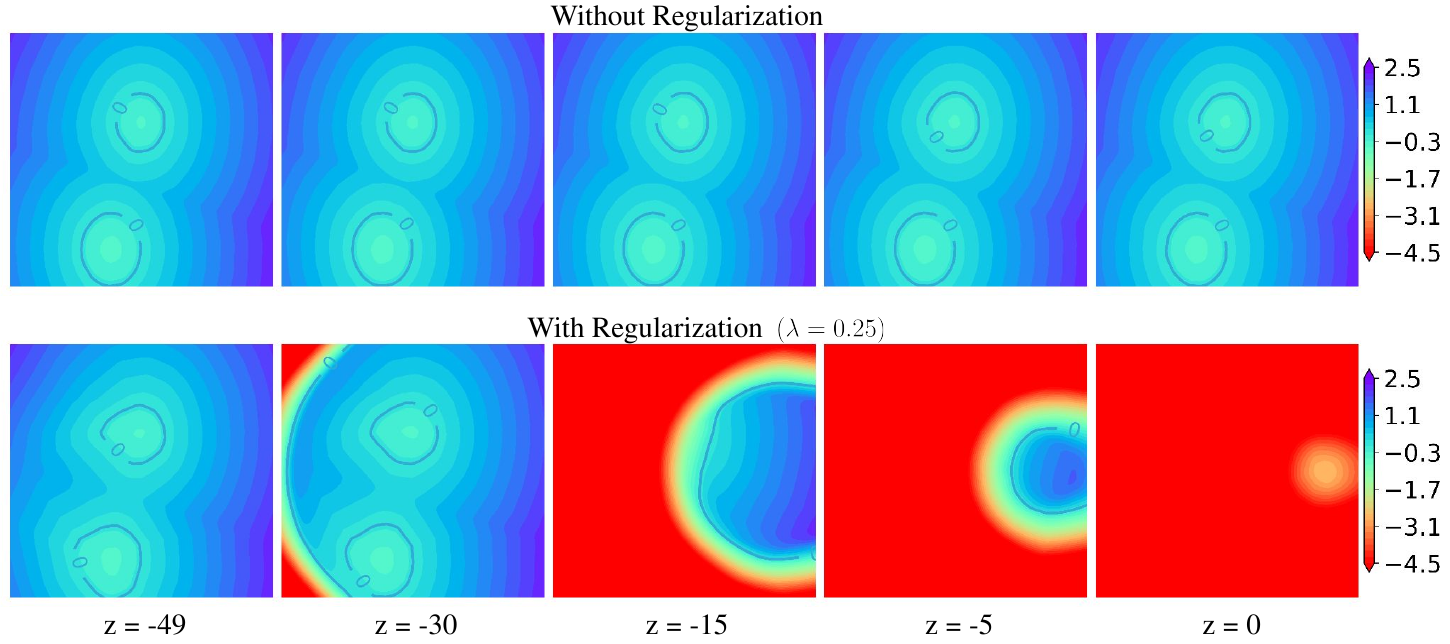}
    \caption{\textbf{Feasibility epigraph value function $\hat{V}_{\theta}(x,z)$}. (\emph{Top}) Learned without the decomposition regularizer, where $\hat{V}_{\theta}$ exhibits weak sensitivity to the epigraph variable $z$ and similar contour structure across different $z$ values. (\emph{Bottom}) Learned with the regularizer ($\lambda=0.25$), which restores meaningful dependence on $z$ and yields distinct level sets consistent with the epigraph formulation.}
    \label{fig:unreg_vs_reg_vhat}
\end{figure}


\subsection{Reward and Safety Value Function Losses}
\label{appendix: reward_safety_losses}

The decomposition regularizer in equation~\eqref{eq:L2_regulariser} requires estimates of the optimal reward value $V_r(x)$ and the optimal safety value $V_s(x)$. We learn these using IQL-style \citep{kostrikov2022offline} Q-function and expectile-distilled value losses. The reward value functions $Q_r$ and $V_r$ are trained via:
\begin{equation}
\label{eq:Qr}
\mathcal{L}_{Q_r}(\phi')=\mathbb{E}\big[( r(x) + \gamma V_r^\phi(x') - Q_r^{\phi'}(x,a) )^2\big],
\end{equation}
\begin{equation}
\label{eq:Vr_expectile}
\mathcal{L}_{V_r}(\phi)=\mathbb{E}\big[\mathcal{L}^\tau( Q_r^{\phi'}(x,a) - V_r^\phi(x) )\big].
\end{equation}

Similarly, the safety value functions $Q_s$ and $V_s$ are learned via:
\begin{equation}
\label{eq:Qs}
\mathcal{L}_{Q_s}(\psi')=\mathbb{E}\big[( \min\{\ell(x),\gamma V_s^\psi(x')\} - Q_s^{\psi'}(x,a) )^2\big],
\end{equation}
\begin{equation}
\label{eq:Vs_expectile}
\mathcal{L}_{V_s}(\psi)=\mathbb{E}\big[\mathcal{L}^\tau( Q_s^{\psi'}(x,a) - V_s^\psi(x) )\big].
\end{equation}

All expectations are empirical averages over minibatches from $\mathcal{D}$, and $\tau$ is the same expectile parameter used for the feasibility value distillation.

\subsection{Expectile Regression}\label{appendix: expectile_regression}
Expectile regression is a classical tool in statistics and econometrics for estimating asymmetric statistics of a random variable. For a random variable $X$, the $\tau$-expectile is defined as the minimizer of the asymmetric least-squares problem
\begin{equation}
m_{\tau}
=\arg\min_{m}
\mathbb{E}_{x\sim X}\left[
L^{\tau}(x - m)
\right],
\end{equation}
where $L_{\tau}(y)=\lvert \tau - \mathbf{1}(y < 0)\rvert y^{2}$.

When $\tau > 0.5$, this loss assigns more weight to samples above the estimate $m_{\tau}$ and less weight to samples below it. Conversely, $\tau < 0.5$ emphasizes lower values. Thus, expectiles interpolate smoothly between the mean ($\tau = 0.5$) and a ``high-value--seeking" statistic as $\tau \rightarrow 1$.

Expectile regression can also be extended to learning conditional expectiles:
\begin{equation}
m_{\tau}(x)
=\arg\min_{m(\cdot)}
\mathbb{E}_{(x,y)\sim\mathcal{D}}
\left[
L^{\tau}(y - m(x))
\right],
\end{equation}
which can be optimized efficiently via stochastic gradient descent. This makes expectiles easy to implement in modern machine-learning pipelines, unlike alternative high-order statistics that require specialized solvers.

\subsection{Algorithm for Learning the Epigraph Feasibility Value Function}

This section covers algorithms corresponding to the procedures explained in Section~\ref{section: method}. The algorithm for learning the Epigraph Feasibility Value Function $\hat{V}_{\theta}(x,z)$ is detailed in Algorithm~\ref{alg: learn_epiflow}.

\begin{algorithm}[ht]
\caption{Learning the Epigraph Feasibility Value Function}
\label{alg: learn_epiflow}
\begin{algorithmic}[1]
\Require Offline dataset $\mathcal{D}=\{(x,a,r,\ell,x')\}$, discount factor $\gamma$, expectile parameter $\tau$, regularization weight $\lambda$
\State Initialize networks $\hat{Q}_{\theta'}$, $\hat{V}_\theta$, $Q_r^{\phi'}$, $V_r^\phi$, $Q_s^{\psi'}$, $V_s^\psi$
\While{not converged}
    \State Sample minibatch $(x,a,r,\ell,x') \sim \mathcal{D}$
    \State Sample $z \sim [z_{\min}, z_{\max}]$ and compute $z' = \frac{z - r(x)}{\gamma}$

    \State \textbf{Feasibility Q-update:}
    \State $y_Q \leftarrow \min\{\ell(x), \gamma \hat{V}_{\theta}(x',z')\}$
    \State Update $\hat{Q}_{\theta'}$ by minimizing $\mathcal{L}_{\hat{Q}}$

    \State \textbf{Reward value updates:}
    \State Update $Q_r^{\phi'}, V_r^{\phi}$ using $\mathcal{L}_{Q_r}, \mathcal{L}_{V_r}$

    \State \textbf{Safety value updates:}
    \State Update $Q_s^{\psi'}, V_s^{\psi}$ using $\mathcal{L}_{Q_s}, \mathcal{L}_{V_s}$

    \State \textbf{Expectile value distillation \& regularization:}
    \State Update $\hat{V}_\theta$ using $\mathcal{L}_{\hat{V}} + \lambda \cdot \mathcal{L}_{reg}$
\EndWhile
\State \Return Learned epigraph feasibility value function $\hat{V}_\theta$
\end{algorithmic}
\end{algorithm}

In succession to the Algorithm~\ref{alg: learn_epiflow}, next we discuss about a practical algorithm to use this learned $\hat{V}_{\theta}(x,z)$ for learning a deployable policy (refer Algorithm~\ref{alg: policy_synthesis}).

\begin{algorithm}[ht]
\caption{Flow Matching based Policy Synthesis}
\label{alg: policy_synthesis}
\begin{algorithmic}[1]
\Require Learned $\hat{V}_{\theta}$, learned $\hat{Q}_{\theta'}$, offline dataset $\mathcal{D}$
\State Initialize flow teacher $\mu_{\zeta}$ and actor $\mu_{\omega}$
\For{each training iteration}
    \State Sample $(x,a) \sim \mathcal{D}$ and $\eta \sim \mathcal{N}(0,I)$
    \State Compute $z^\star(x) \gets \textsc{BinarySearch}(\hat{V}_{\theta}(x,z) \ge 0)$
    \State Update $\mu_{\zeta}$ by minimizing $\mathcal{L}_{\mathrm{flow}}$
    \State Update $\mu_{\omega}$ by minimizing $\mathcal{L}_{\mathrm{actor}}$
\EndFor
\State \Return $\pi^\star(a\mid x) \leftarrow \mu_{\omega}(x,\eta)$
\end{algorithmic}
\end{algorithm}

\begin{figure*}
    \centering
    \includegraphics[width=0.99\linewidth]{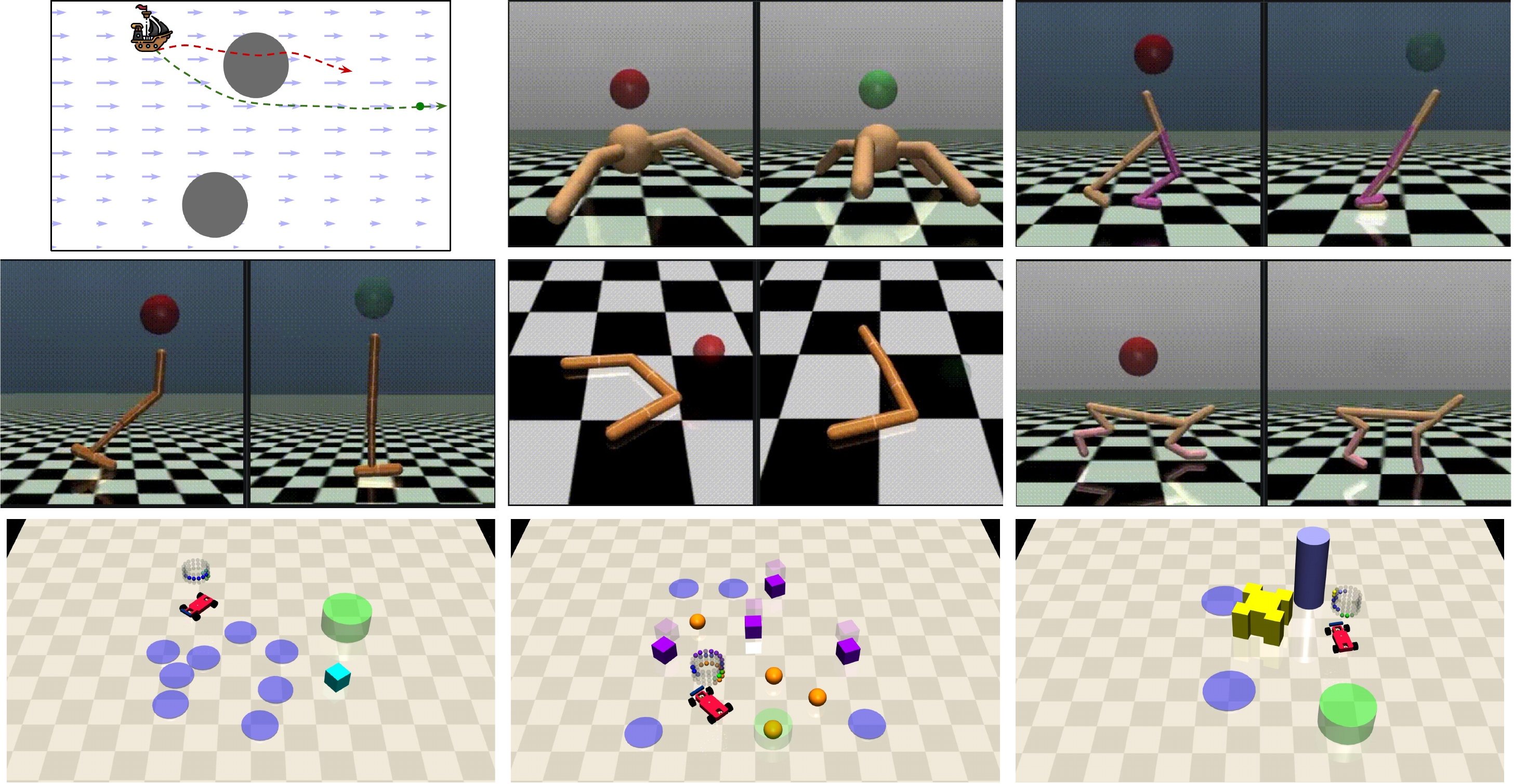}
    \caption{\textbf{Illustration of all evaluation environments.} (\textit{Top-Left}) environment illustrates Boat Navigation with obstacles and a goal point. (\textit{Middle}) illustrates SafeVelocity MuJoCo environments with \textbf{\textcolor{red}{Red}} sphere denoting unsafe state. (\textit{Bottom}) depicts the SafeCarNavigation MuJoCo environments with various obstacles and objectives.}
    \label{fig: all_envs}
\end{figure*}

\section{Description of the Experiments} \label{appendix: env_desc}

\subsection{Boat Navigation} \label{appendix: Boat2D}
The state $x \in \mathcal{X} = [-3, 2]\times[-2, 2]$ of the 2D Boat system are $x = [x_1, x_2]^T$, where, $x_1, x_2$ are the $x$ and $y$ coordinates of the boat respectively. We define the step reward at each step, $r(x)$ as:  
\begin{equation}
 \begin{aligned}
     r(x) = C\cdot\Bigg(-\left\|
    \begin{bmatrix} x_1 \\ x_2 \end{bmatrix}
    -
    \begin{bmatrix} x_{g1} \\ x_{g2} \end{bmatrix}
    \right\| \Bigg)
   \end{aligned}
\end{equation}
where $[x_{g1}, x_{g2}]^T = [1.5, 0.0]^T$ denotes the goal location, and $C = 0.1$. Maximising this reward drives the boat toward the goal. Consequently, the (augmented) dynamics of the 2D Boat system are:
\begin{align*}
    (x_1)_{t+1} &= (x_1)_{t} + (a_1 + 2 - 0.5 x_2^2) \cdot \Delta t \\
    (x_2)_{t+1} &= (x_2)_{t} + a_2 \cdot \Delta t \\
    z_{t+1} &= (z_t - r(x))/\gamma
\end{align*}
where $\Delta t$ represents discrete time step, $a_1, a_2$ represent the velocity control action in $x_1$ and $x_2$ directions respectively, with $a_1^2 + a_2^2 \leq 1$
and $2 - 0.5x_2^2$ specifies the current drift along the $x_1$-axis. The safety constraints are encoded as:  
\begin{align}
    \ell(x) := min ( \|x - (-0.5, 0.5)^T \| - 0.4, \|x - (-1.0, -1.2)^T \| - 0.4)
\end{align}
where $\ell(x) < 0$ indicates that the boat is inside a obstacle, thereby ensuring that the sub-level set of $\ell(x)$ defines the failure region. These dynamics are used only to simulate the environment and generate the dataset; they are never used by our method or the baselines.

\paragraph{Offline Data Generation.}
Since this is a custom environment, we construct an offline dataset which can be used to train a safe and performant policy. We sample $2500$ initial states uniformly from $\mathcal{X}$ and simulate each trajectory for $400$ discrete timesteps with step size $\Delta t = 0.005s$. During data collection, control inputs are drawn uniformly at random from the admissible range, ensuring diverse system trajectories for learning the controller. Following the DSRL procedure \citep{liu2024dsrl}, the reward normalization constants $r_{\min}$ and $r_{\max}$ used for Boat Navigation in Table~\ref{tab:main_results} are the minimum and maximum episode returns (sums of rewards) over the trajectories in this offline dataset: $r_{\min}=-1963.26$ and $r_{\max}=-167.80$.



\subsection{Safety Gymnasium Environments}
To evaluate scalability to higher-dimensional systems we use Safety Gymnasium \cite{ji2023safety}. We consider (i) \emph{safe velocity} variants that impose velocity constraints on Gymnasium MuJoCo-v4 agents (HalfCheetah, Hopper, Ant, Walker2D, Swimmer), and (ii) car navigation tasks (CarButton1, CarGoal1, CarPush1) focused on collision avoidance.

The agent receives a scalar cost signal on the basis of its current state at each step.
This cost can be framed as a sparse and simple safety function
\begin{equation}
\ell(x)=\begin{cases}
25 & \text{if }\mathrm{cost}=0,\\[4pt]
-1 & \text{if }\mathrm{cost}>0.
\end{cases}
\end{equation}
since $\ell(x) \ge 0$ corresponds to safe state. This safety value function provides us a sparse signal.

We additionally evaluate on the eight Bullet-Safety-Gym tasks from DSRL \citep{Gronauer2022BulletSafetyGym, liu2024dsrl}: in the Run tasks, the agent must move fast along a corridor without crossing its boundaries or exceeding a velocity threshold, and in the Circle tasks, it is rewarded for moving along a circle while remaining inside a safe region. Table~\ref{tab:bullet_results} reports the results.

Table~\ref{tab:main_results} reports the complete numerical results for Boat Navigation and Safety Gymnasium summarized in Figure~\ref{fig:results}.

\begin{table}[ht]
\centering
\caption{\textbf{Results on Boat Navigation and Safety Gymnasium.} Normalized reward $R\uparrow$ and normalized cost $C\downarrow$, reported as mean$_{\pm\text{std}}$ over 500 evaluation episodes and 5 seeds. \us{Gray}: unsafe under the DSRL convention ($C\ge 1$). \textbf{Bold}: method with the lowest cost on each task (highest reward among ties).}
\label{tab:main_results}
\setlength{\tabcolsep}{3pt}
\renewcommand{\arraystretch}{1.15}
\resizebox{\linewidth}{!}{%
\begin{tabular}{l cc cc cc cc cc}
\toprule
& \multicolumn{2}{c}{Boat} & \multicolumn{2}{c}{HalfCheetah} & \multicolumn{2}{c}{Hopper} & \multicolumn{2}{c}{Ant} & \multicolumn{2}{c}{Walker2D} \\
\cmidrule(lr){2-3}\cmidrule(lr){4-5}\cmidrule(lr){6-7}\cmidrule(lr){8-9}\cmidrule(lr){10-11}
Method & $R\uparrow$ & $C\downarrow$ & $R\uparrow$ & $C\downarrow$ & $R\uparrow$ & $C\downarrow$ & $R\uparrow$ & $C\downarrow$ & $R\uparrow$ & $C\downarrow$ \\
\midrule
BCQ-Lag & \us{$0.65_{\pm 0.03}$} & \us{$1.89_{\pm 0.80}$} & \us{$1.06_{\pm 0.05}$} & \us{$58.50_{\pm 1.00}$} & \us{$0.72_{\pm 0.02}$} & \us{$19.80_{\pm 1.00}$} & \us{$1.00_{\pm 0.03}$} & \us{$7.40_{\pm 0.60}$} & \us{$0.78_{\pm 0.03}$} & \us{$3.50_{\pm 0.30}$} \\
BEAR-Lag & \us{$0.78_{\pm 0.04}$} & \us{$1.66_{\pm 0.65}$} & \us{$0.97_{\pm 0.06}$} & \us{$15.60_{\pm 0.90}$} & \us{$0.36_{\pm 0.02}$} & \us{$17.30_{\pm 0.90}$} & $-1.01_{\pm 0.01}$ & $0.00_{\pm 0.00}$ & \us{$0.78_{\pm 0.03}$} & \us{$1.70_{\pm 0.20}$} \\
CPQ & \us{$0.75_{\pm 0.04}$} & \us{$7.41_{\pm 0.70}$} & \us{$0.39_{\pm 0.05}$} & \us{$7.00_{\pm 0.60}$} & \us{$0.13_{\pm 0.01}$} & \us{$10.00_{\pm 0.40}$} & $-1.01_{\pm 0.01}$ & $0.00_{\pm 0.00}$ & $0.09_{\pm 0.02}$ & $0.35_{\pm 0.21}$ \\
COptiDICE & \us{$0.51_{\pm 0.03}$} & \us{$1.72_{\pm 0.45}$} & $0.61_{\pm 0.07}$ & $0.00_{\pm 0.00}$ & \us{$0.16_{\pm 0.03}$} & \us{$7.90_{\pm 0.35}$} & \us{$1.00_{\pm 0.03}$} & \us{$13.30_{\pm 0.70}$} & \us{$0.14_{\pm 0.01}$} & \us{$3.80_{\pm 0.25}$} \\
C2IQL & \us{$0.66_{\pm 0.03}$} & \us{$2.36_{\pm 0.50}$} & \us{$1.02_{\pm 0.03}$} & \us{$25.20_{\pm 0.80}$} & \us{$0.16_{\pm 0.01}$} & \us{$3.20_{\pm 0.25}$} & \us{$1.00_{\pm 0.03}$} & \us{$9.60_{\pm 0.60}$} & \us{$0.21_{\pm 0.02}$} & \us{$3.10_{\pm 0.30}$} \\
\addlinespace[3pt]
FISOR & $0.73_{\pm 0.02}$ & $0.86_{\pm 0.30}$ & $0.77_{\pm 0.03}$ & $0.00_{\pm 0.00}$ & $0.02_{\pm 0.00}$ & $0.16_{\pm 0.10}$ & $0.65_{\pm 0.02}$ & $0.00_{\pm 0.00}$ & $0.09_{\pm 0.00}$ & $0.25_{\pm 0.10}$ \\
V-OCBF & $0.59_{\pm 0.01}$ & $0.00_{\pm 0.00}$ & $0.43_{\pm 0.02}$ & $0.00_{\pm 0.00}$ & $-0.04_{\pm 0.00}$ & $0.05_{\pm 0.02}$ & $0.30_{\pm 0.02}$ & $0.00_{\pm 0.00}$ & $\mathbf{0.03}_{\pm 0.00}$ & $\mathbf{0.09}_{\pm 0.08}$ \\
NCBF & $0.55_{\pm 0.01}$ & $0.92_{\pm 0.18}$ & $0.32_{\pm 0.02}$ & $0.03_{\pm 0.01}$ & $-0.09_{\pm 0.00}$ & $0.20_{\pm 0.07}$ & $0.23_{\pm 0.02}$ & $0.03_{\pm 0.01}$ & $0.02_{\pm 0.00}$ & $0.30_{\pm 0.07}$ \\
\textbf{Ours} & $\mathbf{0.82}_{\pm 0.01}$ & $\mathbf{0.00}_{\pm 0.00}$ & $\mathbf{0.93}_{\pm 0.03}$ & $\mathbf{0.00}_{\pm 0.00}$ & $\mathbf{0.15}_{\pm 0.00}$ & $\mathbf{0.05}_{\pm 0.03}$ & $\mathbf{0.81}_{\pm 0.03}$ & $\mathbf{0.00}_{\pm 0.00}$ & $0.10_{\pm 0.00}$ & $0.10_{\pm 0.10}$ \\
\midrule
& \multicolumn{2}{c}{Swimmer} & \multicolumn{2}{c}{CarButton1} & \multicolumn{2}{c}{CarGoal1} & \multicolumn{2}{c}{CarPush1} &  &  \\
\cmidrule(lr){2-3}\cmidrule(lr){4-5}\cmidrule(lr){6-7}\cmidrule(lr){8-9}
Method & $R\uparrow$ & $C\downarrow$ & $R\uparrow$ & $C\downarrow$ & $R\uparrow$ & $C\downarrow$ & $R\uparrow$ & $C\downarrow$ &  &  \\
\midrule
BCQ-Lag & \us{$0.58_{\pm 0.03}$} & \us{$29.10_{\pm 0.80}$} & \us{$0.01_{\pm 0.07}$} & \us{$5.27_{\pm 0.40}$} & \us{$0.46_{\pm 0.05}$} & \us{$2.98_{\pm 0.40}$} & \us{$0.24_{\pm 0.05}$} & \us{$1.37_{\pm 0.15}$} &  &  \\
BEAR-Lag & \us{$0.15_{\pm 0.07}$} & \us{$2.70_{\pm 1.00}$} & \us{$0.21_{\pm 0.09}$} & \us{$8.68_{\pm 0.50}$} & \us{$0.57_{\pm 0.03}$} & \us{$3.38_{\pm 0.25}$} & \us{$0.23_{\pm 0.05}$} & \us{$1.34_{\pm 0.15}$} &  &  \\
CPQ & \us{$0.10_{\pm 0.06}$} & \us{$1.20_{\pm 0.50}$} & \us{$0.25_{\pm 0.14}$} & \us{$44.61_{\pm 1.00}$} & \us{$0.67_{\pm 0.04}$} & \us{$4.18_{\pm 0.30}$} & \us{$-0.00_{\pm 0.06}$} & \us{$4.04_{\pm 0.40}$} &  &  \\
COptiDICE & \us{$0.68_{\pm 0.04}$} & \us{$38.70_{\pm 1.20}$} & \us{$-0.11_{\pm 0.06}$} & \us{$3.03_{\pm 0.35}$} & \us{$0.41_{\pm 0.02}$} & \us{$2.04_{\pm 0.15}$} & \us{$0.21_{\pm 0.06}$} & \us{$1.75_{\pm 0.20}$} &  &  \\
C2IQL & \us{$0.45_{\pm 0.07}$} & \us{$33.60_{\pm 1.10}$} & \us{$0.09_{\pm 0.09}$} & \us{$8.56_{\pm 0.50}$} & $0.02_{\pm 0.01}$ & $0.06_{\pm 0.05}$ & \us{$0.10_{\pm 0.04}$} & \us{$1.02_{\pm 0.10}$} &  &  \\
\addlinespace[3pt]
FISOR & $-0.07_{\pm 0.01}$ & $0.10_{\pm 0.10}$ & $-0.27_{\pm 0.02}$ & $0.44_{\pm 0.10}$ & $0.02_{\pm 0.01}$ & $0.02_{\pm 0.01}$ & $0.10_{\pm 0.03}$ & $0.24_{\pm 0.08}$ &  &  \\
V-OCBF & $-0.13_{\pm 0.00}$ & $0.00_{\pm 0.00}$ & $\mathbf{-0.45}_{\pm 0.02}$ & $\mathbf{0.05}_{\pm 0.03}$ & $-0.01_{\pm 0.00}$ & $0.01_{\pm 0.00}$ & $0.02_{\pm 0.02}$ & $0.06_{\pm 0.03}$ &  &  \\
NCBF & $-0.17_{\pm 0.00}$ & $0.12_{\pm 0.03}$ & $-0.54_{\pm 0.02}$ & $0.48_{\pm 0.07}$ & $-0.03_{\pm 0.00}$ & $0.03_{\pm 0.02}$ & $-0.04_{\pm 0.02}$ & $0.28_{\pm 0.05}$ &  &  \\
\textbf{Ours} & $\mathbf{-0.04}_{\pm 0.00}$ & $\mathbf{0.00}_{\pm 0.00}$ & $-0.14_{\pm 0.02}$ & $0.07_{\pm 0.04}$ & $\mathbf{0.02}_{\pm 0.01}$ & $\mathbf{0.01}_{\pm 0.01}$ & $\mathbf{0.22}_{\pm 0.03}$ & $\mathbf{0.06}_{\pm 0.03}$ &  &  \\
\bottomrule
\end{tabular}}
\end{table}

\begin{table}[ht]
\centering
\caption{\textbf{Results on Bullet-Safety-Gym (DSRL).} Same metrics as Table~\ref{tab:main_results}; 500 evaluation episodes over 5 seeds. \us{Gray}: unsafe under the DSRL convention ($C \ge 1$). \textbf{Bold}: method with the lowest cost on each task (highest reward among ties).}
\label{tab:bullet_results}
\setlength{\tabcolsep}{3pt}
\renewcommand{\arraystretch}{1.15}
\resizebox{\linewidth}{!}{%
\begin{tabular}{l cc cc cc cc}
\toprule
& \multicolumn{2}{c}{Ant-Run} & \multicolumn{2}{c}{Ball-Run} & \multicolumn{2}{c}{Car-Run} & \multicolumn{2}{c}{Drone-Run} \\
\cmidrule(lr){2-3}\cmidrule(lr){4-5}\cmidrule(lr){6-7}\cmidrule(lr){8-9}
Method & $R\uparrow$ & $C\downarrow$ & $R\uparrow$ & $C\downarrow$ & $R\uparrow$ & $C\downarrow$ & $R\uparrow$ & $C\downarrow$ \\
\midrule
BCQ-Lag & \us{$0.63_{\pm 0.07}$} & \us{$4.65_{\pm 0.85}$} & \us{$0.31_{\pm 0.05}$} & \us{$2.28_{\pm 0.55}$} & $0.97_{\pm 0.04}$ & $0.10_{\pm 0.30}$ & \us{$0.64_{\pm 0.08}$} & \us{$10.32_{\pm 1.55}$} \\
BEAR-Lag & $0.07_{\pm 0.03}$ & $0.00_{\pm 0.00}$ & \us{$-0.36_{\pm 0.06}$} & \us{$9.10_{\pm 1.40}$} & $0.66_{\pm 0.05}$ & $0.03_{\pm 0.05}$ & \us{$0.15_{\pm 0.04}$} & \us{$4.31_{\pm 0.80}$} \\
CPQ & $-0.02_{\pm 0.02}$ & $0.00_{\pm 0.00}$ & \us{$0.27_{\pm 0.05}$} & \us{$3.70_{\pm 0.75}$} & \us{$1.01_{\pm 0.06}$} & \us{$9.40_{\pm 1.35}$} & \us{$0.49_{\pm 0.06}$} & \us{$1.69_{\pm 0.42}$} \\
COptiDICE & $0.51_{\pm 0.01}$ & $0.03_{\pm 0.01}$ & $0.42_{\pm 0.08}$ & $0.15_{\pm 0.03}$ & $0.81_{\pm 0.02}$ & $0.00_{\pm 0.00}$ & $0.57_{\pm 0.01}$ & $0.17_{\pm 0.00}$ \\
C2IQL & $0.67_{\pm 0.01}$ & $0.11_{\pm 0.01}$ & $0.53_{\pm 0.05}$ & $0.11_{\pm 0.01}$ & $0.85_{\pm 0.01}$ & $0.01_{\pm 0.01}$ & $0.64_{\pm 0.02}$ & $0.09_{\pm 0.02}$ \\
\addlinespace[3pt]
FISOR & $0.16_{\pm 0.03}$ & $0.00_{\pm 0.00}$ & $0.22_{\pm 0.01}$ & $0.00_{\pm 0.00}$ & $0.72_{\pm 0.01}$ & $0.00_{\pm 0.00}$ & $0.18_{\pm 0.05}$ & $0.05_{\pm 0.02}$ \\
V-OCBF & $0.15_{\pm 0.03}$ & $0.00_{\pm 0.00}$ & $0.20_{\pm 0.01}$ & $0.00_{\pm 0.00}$ & $0.66_{\pm 0.01}$ & $0.00_{\pm 0.00}$ & $0.17_{\pm 0.05}$ & $0.00_{\pm 0.00}$ \\
\textbf{Ours} & $\mathbf{0.33}_{\pm 0.01}$ & $\mathbf{0.00}_{\pm 0.00}$ & $\mathbf{0.33}_{\pm 0.01}$ & $\mathbf{0.00}_{\pm 0.00}$ & $\mathbf{0.95}_{\pm 0.01}$ & $\mathbf{0.00}_{\pm 0.00}$ & $\mathbf{0.43}_{\pm 0.01}$ & $\mathbf{0.00}_{\pm 0.00}$ \\
\midrule
& \multicolumn{2}{c}{Ant-Circle} & \multicolumn{2}{c}{Ball-Circle} & \multicolumn{2}{c}{Car-Circle} & \multicolumn{2}{c}{Drone-Circle} \\
\cmidrule(lr){2-3}\cmidrule(lr){4-5}\cmidrule(lr){6-7}\cmidrule(lr){8-9}
Method & $R\uparrow$ & $C\downarrow$ & $R\uparrow$ & $C\downarrow$ & $R\uparrow$ & $C\downarrow$ & $R\uparrow$ & $C\downarrow$ \\
\midrule
BCQ-Lag & \us{$0.54_{\pm 0.06}$} & \us{$3.92_{\pm 0.70}$} & \us{$0.76_{\pm 0.05}$} & \us{$3.99_{\pm 0.72}$} & \us{$0.62_{\pm 0.06}$} & \us{$4.15_{\pm 0.78}$} & \us{$0.62_{\pm 0.06}$} & \us{$1.96_{\pm 0.45}$} \\
BEAR-Lag & \us{$0.73_{\pm 0.05}$} & \us{$9.90_{\pm 1.50}$} & \us{$0.79_{\pm 0.05}$} & \us{$4.87_{\pm 0.90}$} & \us{$0.74_{\pm 0.05}$} & \us{$3.34_{\pm 0.62}$} & \us{$0.72_{\pm 0.05}$} & \us{$5.43_{\pm 0.95}$} \\
CPQ & $0.00_{\pm 0.02}$ & $0.00_{\pm 0.00}$ & $0.62_{\pm 0.05}$ & $0.87_{\pm 0.40}$ & \us{$0.70_{\pm 0.05}$} & \us{$1.05_{\pm 0.43}$} & \us{$-0.22_{\pm 0.04}$} & \us{$1.74_{\pm 0.40}$} \\
COptiDICE & $0.17_{\pm 0.05}$ & $0.11_{\pm 0.03}$ & $0.56_{\pm 0.03}$ & $0.12_{\pm 0.01}$ & $0.39_{\pm 0.02}$ & $0.14_{\pm 0.02}$ & $0.36_{\pm 0.01}$ & $0.05_{\pm 0.01}$ \\
C2IQL & $0.59_{\pm 0.05}$ & $0.09_{\pm 0.02}$ & $0.65_{\pm 0.03}$ & $0.10_{\pm 0.01}$ & $0.67_{\pm 0.01}$ & $0.11_{\pm 0.01}$ & $0.70_{\pm 0.02}$ & $0.10_{\pm 0.01}$ \\
\addlinespace[3pt]
FISOR & $0.10_{\pm 0.03}$ & $0.00_{\pm 0.00}$ & $0.23_{\pm 0.02}$ & $0.00_{\pm 0.00}$ & $0.21_{\pm 0.05}$ & $0.00_{\pm 0.00}$ & $0.41_{\pm 0.03}$ & $0.00_{\pm 0.00}$ \\
V-OCBF & $0.09_{\pm 0.03}$ & $0.00_{\pm 0.00}$ & $0.21_{\pm 0.02}$ & $0.00_{\pm 0.00}$ & $0.19_{\pm 0.05}$ & $0.00_{\pm 0.00}$ & $0.38_{\pm 0.03}$ & $0.00_{\pm 0.00}$ \\
\textbf{Ours} & $\mathbf{0.27}_{\pm 0.01}$ & $\mathbf{0.00}_{\pm 0.00}$ & $\mathbf{0.57}_{\pm 0.01}$ & $\mathbf{0.00}_{\pm 0.00}$ & $\mathbf{0.56}_{\pm 0.01}$ & $\mathbf{0.00}_{\pm 0.00}$ & $\mathbf{0.57}_{\pm 0.01}$ & $\mathbf{0.00}_{\pm 0.00}$ \\
\bottomrule
\end{tabular}}
\end{table}

\section{Ablation Studies}\label{appendix: ablation_study}
We conduct a series of ablation studies to isolate and validate the key design choices of our method. The primary testbed for these analyses is the \emph{Safe Boat Navigation} environment, which provides a controlled yet challenging safety-critical setting. Its low-dimensional, 2D structure enables clear visualization and interpretation of policy behavior, while the presence of nonlinear river drift dynamics makes safe navigation nontrivial for existing baselines. While most ablations are centered on the Boat environment, we additionally include select higher-dimensional systems in specific studies to assess robustness and generalization beyond the primary domain. The following subsections detail these analyses.

\subsection{\texorpdfstring{Sensitivity to Expectile parameter ($\tau$)}{Sensitivity to Expectile parameter (tau)}} \label{appendix: tau-sensitivity}

As described in Section~\ref{section: method}, we employ a $\tau$-expectile regression objective to learn the epigraph feasibility value function $\hat{V}(x,z)$, which is subsequently used for policy extraction in the SC-MDP setting. When $\tau = 0.5$, the expectile loss reduces to a symmetric mean-squared error objective, yielding an estimate that reflects the average safety and performance induced by the behavior policy. This \emph{behavior-induced} value function is typically overly conservative and fails to emphasize high-quality safe transitions, which is misaligned with the objectives of safety-constrained control.

To study the effect of asymmetric weighting, we compare $\tau = 0.5$ against larger expectile values that place increasing emphasis on the upper tail of the value distribution. Table~\ref{tab: tau_analysis} reports mean episode reward, safety rate, and average episode cost on the Boat Navigation environment for different $\tau$ settings.

\begin{table}[!ht]
\centering
\small
\caption{Mean Reward, Safety Rate (\%), and Mean Cost reported for Boat Navigation by varying $\tau$ in the expectile regression. ($\pm$) represents standard deviation.}
\label{tab: tau_analysis}
\begin{tabular}{lcccc}
\toprule
\textbf{Expectile Parameter} $(\boldsymbol{\tau})$ & \textbf{0.5} & \textbf{0.7} & \textbf{0.8} & \textbf{0.9} \\
& \textbf{(Behaviour Policy)} & & & \\
\midrule
\textbf{Mean Episode Reward} & -485.86 $\pm$ 0.74 & -492.36 $\pm$ 0.21 & -496.01 $\pm$ 0.58 & \textbf{-496.10} $\pm$ \textbf{0.37} \\
\midrule
\textbf{Safety Rate (\%)} & 88.9 $\pm$ 0.23 & 98.9 $\pm$ 1.01 & 99.9 $\pm$ 0.18 & \textbf{100.0} $\pm$ \textbf{0.0} \\
\midrule
\textbf{Mean Episode Cost} & 1.95 $\pm$ 0.19 & 0.09 $\pm$ 0.03 & 0.04 $\pm$ 0.05 & \textbf{0.0} $\pm$ \textbf{0.0} \\
\bottomrule
\end{tabular}
\end{table}


We observe that setting $\tau = 0.5$ leads to the lowest safety rate, as the resulting value function underweights safety-critical regions that are infrequently represented in the offline dataset. As $\tau$ increases, safety performance improves steadily while reward remains competitive, reflecting stronger alignment with high-value demonstrated actions. Consequently, larger expectile values produce policies that more effectively capture the upper envelope of safe and performant behaviors.

In idealized settings, such as deterministic dynamics with minimal dataset noise, pushing $\tau \rightarrow 1$ concentrates the estimator on the extreme upper tail and can lead to the strongest empirical performance by prioritizing high-value next-state targets. However, this regime is brittle in practice as rare suboptimal actions or stochastic disturbances may generate overly optimistic transitions, causing extreme expectile estimators to overfit spurious outcomes.

Balancing this trade-off, we fix $\tau = 0.9$ in all experiments. This choice substantially up-weights high-value targets while preserving robustness to noise and outliers commonly present in offline datasets.

\subsection{Resilience to External Perturbations} \label{appendix: robust_control}
\begin{figure}[ht]
    \centering
    \includegraphics[width=\linewidth]{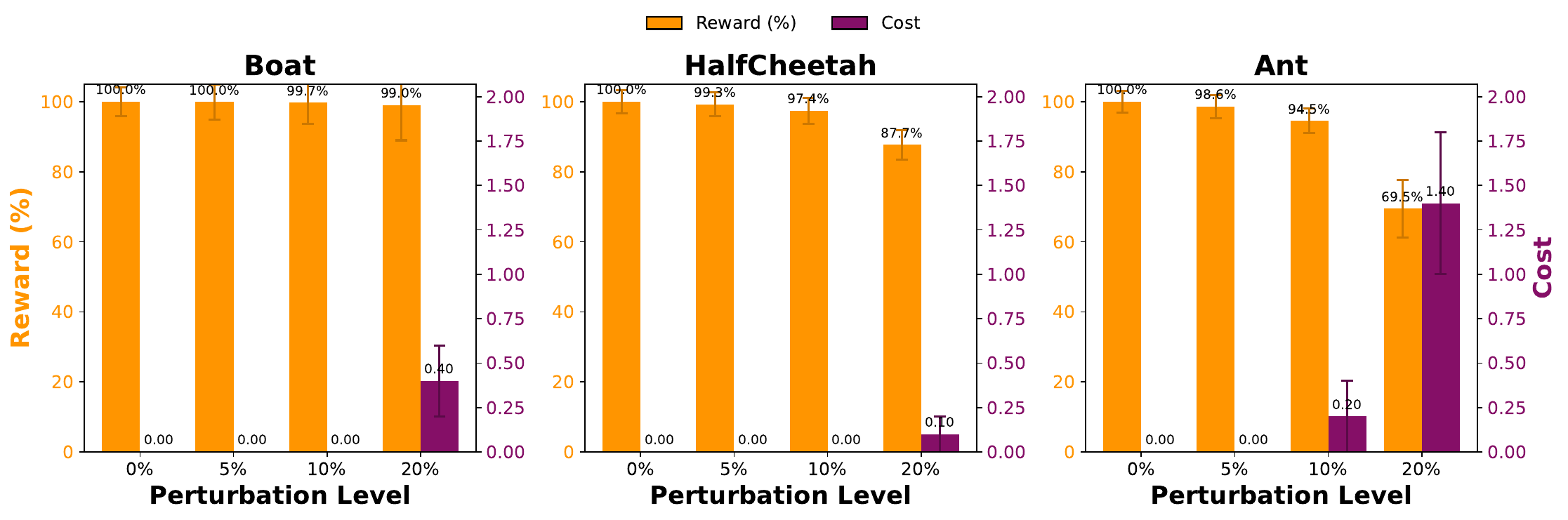}
    \caption{\textbf{Results for Control under Perturbations.} The plots show average reward (\%) ($\pm$ standard deviation) achieved relative to no perturbations case, at different levels of perturbations (5\%, 10\% and 20\%), and their average absolute cost values ($\pm$ standard deviation) respectively.}
    \label{fig: perturb_analysis}
\end{figure}

We next evaluate the robustness of the recovered policy to perturbations in the executed control action, which indirectly probes sensitivity to errors in the learned feasibility value function. Specifically, we perturb the selected optimal action $a$ by additive zero-mean Gaussian noise $\mathcal{N}(0,\sigma^2)$, with $\sigma$ scaled relative to the maximum admissible action magnitude. We consider three perturbation levels: $5\%$, $10\%$, and $20\%$ of the action bound, and evaluate performance over 100 independent rollouts per setting, using five random seeds. In addition to the Boat Navigation task, we also include two higher-dimensional systems from Safety Gymnasium, namely Safe Velocity HalfCheetah and Ant, to assess robustness beyond the primary environment.

Figure~\ref{fig: perturb_analysis} summarizes reward and safety cost trends across perturbation levels. Across all environments, safety costs increase only mildly under small to moderate perturbations, indicating that the learned policies tolerate reasonable inaccuracies in control execution even in safety-critical regimes. While rewards degrade more noticeably at larger perturbation magnitudes, this degradation becomes significant only under aggressive noise injection. These results suggest that precise action execution improves performance, but our framework retains strong safety-aware behavior and competitive returns under moderate control uncertainty. Future work will investigate uncertainty-aware and robust optimization formulations that explicitly model and mitigate the effects of disturbances.

\subsection{\texorpdfstring{Effect of the Regularizer Parameter $\lambda$ on $\hat{V}_{\theta}$}{Effect of the Regularizer Parameter lambda on V-hat}}
\label{appendix:reg_ablation}

We continue from the ablation study on regularizer parameter in section \ref{subsec: abl_std} of the main paper. Figure~\ref{fig: regularizer_analysis} visualizes the inferred per-state budget $z^\ast(x)$ (equation \eqref{eq: z_opt_search}) for the Boat Navigation task. Based on the figure there seems a bias--variance like trade-off across the sweep. For small regularizer parameter value ($\lambda=0.1$) the decomposition constraint is too weak as $\hat{V}_{\theta}$ still exhibits low sensitivity to $z$, causing inflated estimates of high reward feasibility (red region towards left-end in $\lambda = 0.1$ plot), which means that even for realistically low reward yielding regions, $\hat{V}_\theta$ suggests higher reward budget $z$ than is possible. However, as $\lambda$ increases, the decomposition term more strongly enforces the upper bound $\hat{V}\le\min\{V_r-z,\,V_s\}$, which brings the $z$-budget over-estimates within reasonable ranges. But increasing the $\lambda$ too much can also drift $\hat{V}_\theta$ away from the original recursion objective. 

\begin{figure}[ht]
    \centering
    \includegraphics[width=0.98\linewidth]{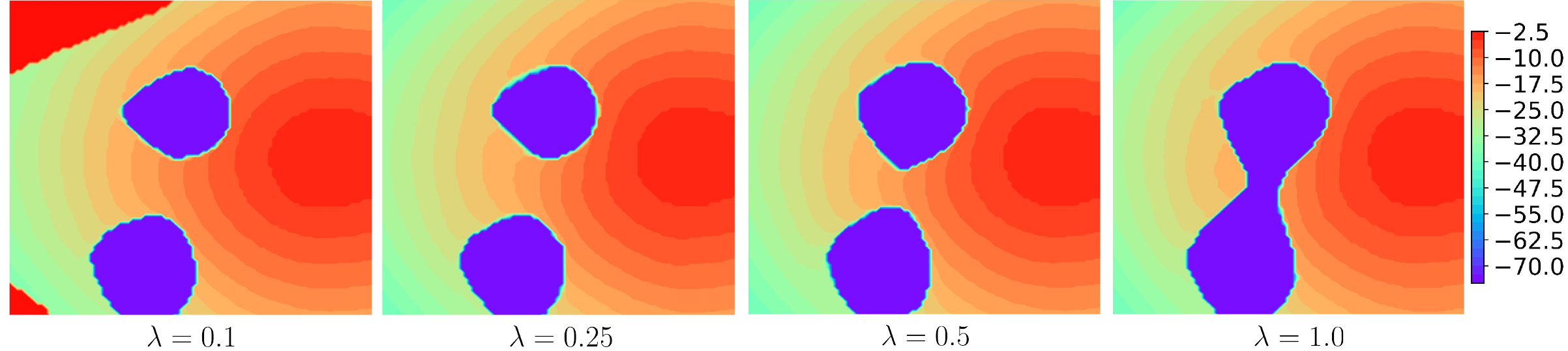}
    \caption{\textbf{Effect of the regularizer parameter $\lambda$ on the inferred epigraph budget.} For the Boat Navigation environment, each panel shows the state-wise optimal budget $z^{\ast}(x)$ defined by the feasibility condition $\hat{V}_{\theta}(x,z^{\ast}) \ge 0$. As $\lambda$ increases, the feasible region contracts and $z^{\ast}(x)$ becomes more conservative; states shaded in purple correspond to infeasible regions.}
    \label{fig: regularizer_analysis}
\end{figure}

Based on the results from the plot in Figure~\ref{fig: lambda_analysis} and balancing the above discussed effects, we select $\lambda=0.25$ for the main experiments as it sufficiently prevents the $z$-desensitization and budget-collapse as in case no regularization, while also avoiding excessive drifting away from original objective as seen at higher weights. Practitioners should tune $\lambda$ on a validation set (a useful range is roughly $0.1$--$0.5$) and inspect both sensitivity diagnostics (e.g., variation of $\hat{V}_\theta$ across $z$ (Appendix~\ref{appendix: eff_reg})) and downstream trade-offs between reward and safety when selecting the final value.

\subsection{Flow Matching v/s Diffusion Model}
In Section~\ref{ssec: policy_syn}, where we discussed about synthesizing our policy using Flow Q-Learning \citep{park2025fql}, a question arises, whether the performance gains in our proposed framework just due to the use of Flow Matching based policy learning or is the use of Epigraph Feasibility Value Function actually adding any value. To study this, we perform the analysis where we compare the results of our proposed framework with 2 more policy learning Approaches, Diffusion Q-Learning \citep{wang2023diffusionql} and Weighted Behavior Cloning.

\begin{table}[!ht]
\centering
\renewcommand{\arraystretch}{1.5}
\setlength{\tabcolsep}{12pt}
\caption{Mean Episode Reward (R) and Mean Episode Cost (C) ($\pm$ standard deviation) reported for proposed framework across different policy backbones. \textbf{Bold}: method with the lowest cost (highest reward among ties).}
\label{tab: fm_vs_ddpm}
\begin{tabular}{cccc}
\toprule
\textbf{Metric} & \textbf{Flow Matching} & \textbf{Diffusion Model} & \textbf{Weighted BC} \\
\midrule
R & \textbf{-496.23} $\pm$ 0.34 & -584.53 $\pm$ 2.37 & -716.60 $\pm$ 5.23 \\
\midrule
C & \textbf{0.0} $\pm$ 0.0 & 0.0 $\pm$ 0.0 & 18.56 $\pm$ 4.45 \\
\bottomrule
\end{tabular}
\end{table}

Table~\ref{tab: fm_vs_ddpm} shows results for these policy learning approaches when plugged into our framework. It is evident that although weighted Behavior Cloning fails miserably to produce good results due to lack of its expressiveness \citep{wang2023diffusionql, park2025fql}, but when Diffusion Q-Learning is implemented, the results still show an improvement compared to the other baselines mentioned in the main paper (Figure~\ref{fig:results}). Overall, Flow Q-Learning based policy outperforms Diffusion Q-Learning in terms of performance, which is also why we prefer to use Flow Q-Learning for our framework.

Conversely, to isolate the contribution of the epigraph value formulation, we keep the flow-matching extractor fixed and replace our value function with FISOR's reachability-based safety value (FISOR-FM). As Table~\ref{tab:fisor_fm} shows, FISOR-FM attains lower return than our method on 7 of 9 tasks, with no task improving, and higher cost on 7 of 9; it in turn improves on FISOR, so both components contribute. Table~\ref{tab:alpha_sweep} sweeps the distillation coefficient $\alpha$ on HalfCheetah for both methods: return varies within the across-seed standard deviation, while cost rises for $\alpha\ge 1$ in both, so the comparison is not driven by the strength of the distillation term. We use $\alpha=0.1$ in all experiments.

\begin{table}[!ht]
\centering
\small
\caption{\textbf{Ablation isolating the epigraph value formulation.} Normalized reward $R\uparrow$ and cost $C\downarrow$ (mean$_{\pm\text{std}}$) over 500 evaluation episodes and 5 seeds. \textbf{Bold}: method with the lowest cost on each task (highest reward among ties).}
\label{tab:fisor_fm}
\setlength{\tabcolsep}{5pt}
\begin{tabular}{l cc cc cc}
\toprule
& \multicolumn{2}{c}{\textbf{Ours}} & \multicolumn{2}{c}{FISOR-FM} & \multicolumn{2}{c}{FISOR} \\
\cmidrule(lr){2-3}\cmidrule(lr){4-5}\cmidrule(lr){6-7}
Task & $R\uparrow$ & $C\downarrow$ & $R\uparrow$ & $C\downarrow$ & $R\uparrow$ & $C\downarrow$ \\
\midrule
Boat & $\mathbf{0.82}_{\pm 0.01}$ & $\mathbf{0.00}_{\pm 0.00}$ & $0.74_{\pm 0.02}$ & $0.13_{\pm 0.06}$ & $0.73_{\pm 0.02}$ & $0.86_{\pm 0.30}$ \\
HalfCheetah & $\mathbf{0.93}_{\pm 0.03}$ & $\mathbf{0.00}_{\pm 0.00}$ & $0.84_{\pm 0.03}$ & $0.00_{\pm 0.00}$ & $0.77_{\pm 0.03}$ & $0.00_{\pm 0.00}$ \\
Hopper & $\mathbf{0.15}_{\pm 0.00}$ & $\mathbf{0.05}_{\pm 0.03}$ & $0.06_{\pm 0.00}$ & $0.14_{\pm 0.10}$ & $0.02_{\pm 0.00}$ & $0.16_{\pm 0.10}$ \\
Ant & $\mathbf{0.81}_{\pm 0.03}$ & $\mathbf{0.00}_{\pm 0.00}$ & $0.72_{\pm 0.02}$ & $0.00_{\pm 0.00}$ & $0.65_{\pm 0.02}$ & $0.00_{\pm 0.00}$ \\
Walker2D & $\mathbf{0.10}_{\pm 0.00}$ & $\mathbf{0.10}_{\pm 0.10}$ & $0.10_{\pm 0.00}$ & $0.11_{\pm 0.10}$ & $0.09_{\pm 0.00}$ & $0.25_{\pm 0.10}$ \\
Swimmer & $\mathbf{-0.04}_{\pm 0.00}$ & $\mathbf{0.00}_{\pm 0.00}$ & $-0.06_{\pm 0.05}$ & $0.06_{\pm 0.01}$ & $-0.07_{\pm 0.01}$ & $0.10_{\pm 0.10}$ \\
CarButton1 & $\mathbf{-0.14}_{\pm 0.02}$ & $\mathbf{0.07}_{\pm 0.04}$ & $-0.23_{\pm 0.02}$ & $0.33_{\pm 0.10}$ & $-0.27_{\pm 0.02}$ & $0.44_{\pm 0.10}$ \\
CarGoal1 & $\mathbf{0.02}_{\pm 0.01}$ & $\mathbf{0.01}_{\pm 0.01}$ & $0.02_{\pm 0.01}$ & $0.02_{\pm 0.01}$ & $0.02_{\pm 0.01}$ & $0.02_{\pm 0.01}$ \\
CarPush1 & $\mathbf{0.22}_{\pm 0.03}$ & $\mathbf{0.06}_{\pm 0.03}$ & $0.13_{\pm 0.03}$ & $0.19_{\pm 0.09}$ & $0.10_{\pm 0.03}$ & $0.24_{\pm 0.08}$ \\
\bottomrule
\end{tabular}
\end{table}

\begin{table}[!ht]
\centering
\small
\caption{\textbf{Sensitivity to $\alpha$} on Safe Velocity HalfCheetah. Normalized reward $R\uparrow$ and cost $C\downarrow$ (mean$_{\pm\text{std}}$) over 500 evaluation episodes and 5 seeds. \textbf{Bold}: method with the lowest cost for each $\alpha$ (highest reward among ties).}
\label{tab:alpha_sweep}
\begin{tabular}{c cc cc}
\toprule
& \multicolumn{2}{c}{\textbf{Ours}} & \multicolumn{2}{c}{FISOR-FM} \\
\cmidrule(lr){2-3}\cmidrule(lr){4-5}
$\alpha$ & $R\uparrow$ & $C\downarrow$ & $R\uparrow$ & $C\downarrow$ \\
\midrule
0.01 & $\mathbf{0.92}_{\pm 0.02}$ & $\mathbf{0.00}_{\pm 0.00}$ & $0.80_{\pm 0.02}$ & $0.00_{\pm 0.00}$ \\
0.1  & $\mathbf{0.93}_{\pm 0.03}$ & $\mathbf{0.00}_{\pm 0.00}$ & $0.84_{\pm 0.03}$ & $0.00_{\pm 0.00}$ \\
1.0  & $\mathbf{0.94}_{\pm 0.03}$ & $\mathbf{0.04}_{\pm 0.01}$ & $0.86_{\pm 0.04}$ & $0.04_{\pm 0.01}$ \\
10   & $\mathbf{0.94}_{\pm 0.02}$ & $\mathbf{0.05}_{\pm 0.01}$ & $0.87_{\pm 0.04}$ & $0.06_{\pm 0.02}$ \\
\bottomrule
\end{tabular}
\end{table}

\section{Experimental Details}\label{appendix: exp_details}
\subsection{Experimental Hardware}
To keep the evaluation fair and avoid any discriminatory added advantage to any specific experiment, all experiments were conducted on a single system equipped with an 14th Gen Intel Core i9-14900KS CPU, 128GB RAM, and an NVIDIA GeForce RTX 4090 GPU for training and experiment evaluations.

\begin{table}[ht]
    \caption{Hyperparameters for the Algorithm.}
    \centering
    \begin{tabular}{lc}
        \hline
        \textbf{Hyperparameter} & \textbf{Value} \\
        \hline
        Value Function Network (VFN) & Multi-Layer Perceptron (MLP) \\
        Activation Function & ReLU \\
        Optimizer & Adam optimizer \\
        Learning Rate & $3\times 10^{-4}$ \\
        Discount Factor ($\gamma$) & 0.99 \\
        Time Step Intervals (FM Policy)($\mu_\zeta$) & 5 (refer Flow Q-Learning) \\
        \hline
        \textbf{Boat Navigation} & \\
        \hline
        Number of Layers in VFN & 2 \\
        Hidden Layer Size in VFN & 256 neurons per layer \\
        Dataset Size & 1M \\
        \hline
        \textbf{Safe Velocity Hopper} & \\
        \hline
        Number of Hidden Layers in VFN & 3 \\
        Hidden Layer Size in VFN & 256 neurons per layer \\
        Dataset Size & 1.32M \\
        \hline
        \textbf{Safe Velocity Half-Cheetah} & \\
        \hline
        Number of Hidden Layers in VFN & 3 \\
        Hidden Layer Size in VFN & 128 neurons per layer \\
        Dataset Size & 249K \\
        \hline
        \textbf{Safe Velocity Ant} & \\
        \hline
        Number of Hidden Layers in VFN & 3 \\
        Hidden Layer Size in VFN & 256 neurons per layer \\
        Dataset Size & 2.09M \\
        \hline
        \textbf{Safe Velocity Walker2D} & \\
        \hline
        Number of Hidden Layers in VFN & 3 \\
        Hidden Layer Size in VFN & 256 neurons per layer \\
        Dataset Size & 2.12M \\
        \hline
        \textbf{Safe Velocity Swimmer} & \\
        \hline
        Number of Hidden Layers in VFN & 3 \\
        Hidden Layer Size in VFN & 256 neurons per layer \\
        Dataset Size & 1.68M \\
        \hline
        \textbf{Safe Car Button} & \\
        \hline
        Number of Hidden Layers in VFN & 3 \\
        Hidden Layer Size in VFN & 256 neurons per layer \\
        Dataset Size & 2.65M \\
        \hline
        \textbf{Safe Car Goal} & \\
        \hline
        Number of Hidden Layers in VFN & 3 \\
        Hidden Layer Size in VFN & 256 neurons per layer \\
        Dataset Size & 1.67M \\
        \hline
        \textbf{Safe Car Push} & \\
        \hline
        Number of Hidden Layers in VFN & 3 \\
        Hidden Layer Size in VFN & 256 neurons per layer \\
        Dataset Size & 2.87M \\
        \hline
    \end{tabular}
    \vspace{-1em}
    \label{tab:training_details}
\end{table}

\subsection{Network Architecture and Training Details of the Proposed Algorithm}
We have compiled and listed down all the hyperparameters that we used to perform our experiments and report the results. These training settings for all the  environments are detailed in the Table~\ref{tab:training_details}. For the MuJoCo environments, we use the widely accepted DSRL \cite{liu2024dsrl} dataset. 

During the training, we set the $\tau$ for expectile regression in Appendix~\ref{appendix: tau-sensitivity} to 0.9. And we use clipped double Q-learning \citep{fujimoto2018addressing}, taking a minimum of two Q. We update the target Q networks using Exponential Moving Average (EMA) where the weight to the new parameters is set to 0.001. Following \citep{kostrikov2022offline}, we clip exponential advantages to $(-\infty, 100]$ in feasible part and $(-\infty, 150]$ in infeasible part. For our paper, we used flow matching implementation from FQL \citep{park2025fql}.

\vspace{7em}
\subsection{Hyperparameters for the Baselines} 

Hyperparameters for the baselines are provided in Table~\ref{tab:baselines_special_nets}. We use the official implementations of BEAR-Lag, BCQ-Lag, CPQ and COptiDICE from \citep{liu2024dsrl}, of C2IQL from \citep{liu2025ciql} of FISOR from \citep{zheng2024fisor} of V-OCBF from \citep{tayal2026vocbf} and of NCBF from \citep{Alexander2020NCBF}, and train all methods on the same DSRL datasets.

\begin{table}[!htbp]
\centering
\small
\caption{Detailed Hyperparameters for Baseline Special Networks. (Layers, Units) notation refers to hidden layers and units per layer. All baselines utilize the DSRL dataset standards \cite{liu2024dsrl}.}
\label{tab:baselines_special_nets}

\setlength{\tabcolsep}{4pt}
\renewcommand{\arraystretch}{1.1}

\begin{tabularx}{\linewidth}{>{\raggedright\arraybackslash}p{0.34\linewidth} >{\raggedright\arraybackslash}p{0.30\linewidth} >{\raggedright\arraybackslash}p{0.32\linewidth}}
\hline
\textbf{Baseline} & \textbf{Network Component} & \textbf{Architecture Specification} \\
\hline
\textbf{C2IQL} \citep{liu2025ciql} & Actor / Critic / Value Nets & MLP, (256, 2) \\
& \textbf{Cost Reconstruction Model} & \textbf{MLP, (5 layers, 512 units each)} \\
& Reward / Cost Advantage & MLP, (256, 2) \\
\hline
\textbf{BCQ-Lag} (Lag. dual of \citep{fujimoto2019off}) & \textbf{VAE (Encoder/Decoder)} & \textbf{MLP, (750, 2), Latent Dim: 2 $\times$ Action Dim} \\
& Perturbation Model $\xi$ & MLP, (400, 2) \\
& Double Q-Critics & MLP, (400, 2) \\
\hline
\textbf{CPQ} \citep{xu2022constraints} & Actor (Policy) Net & MLP, (256, 2) \\
& \textbf{Constraint-Penalized Q} & \textbf{Ensemble DoubleQCritic, (256, 2)} \\
& Cost Critic Net & MLP, (256, 2) \\
\hline
\textbf{FISOR} \citep{zheng2024fisor} & Diffusion Denoiser (Actor) & DiffusionDenoiserMLP, (256, 2) \\
& \textbf{Feasibility Classifier} & \textbf{MLP, (256, 2), Sigmoid Output} \\
& Energy/Value Guidance & MLP, (256, 2) \\
\hline
\textbf{COptiDICE} \citep{lee2022coptidice} & \textbf{Dual / $\nu$ Network} & \textbf{DualNet (Value-like), (256, 2)} \\
& Actor (Extraction) Net & MLP, (256, 2) \\
\hline
\textbf{BEAR-Lag} (Lag. dual of \citep{kumar2019stabilizing}) & \textbf{VAE (Support Model)} & \textbf{MLP, (750, 2)} \\
& Actor (Policy) Net & SquashedGaussianMLPActor, (256, 2) \\
& Cost / Reward Critics & MLP, (256, 2) \\
\hline
\textbf{V-OCBF} \citep{tayal2026vocbf} & Value Guided Neural CBF & MLP, (256, 2) \\
\hline
\textbf{NCBF} \citep{Alexander2020NCBF} & Neural CBF & MLP, (256, 2) \\
\hline
\end{tabularx}
\end{table}

\subsection{Code}
\label{appendix: code}
The code for EpiFlow is available at \url{https://github.com/tau-intelligence/EpiFlow}.

\section{Broader Impact}
\label{appendix: broader_impact}
Our framework learns safe and performant policies from fixed datasets, a problem relevant to many safety-critical applications such as robotics, autonomous driving, and aerial systems, where online trial-and-error exploration is costly or dangerous. By respecting state constraints while retaining task performance, it brings data-driven controllers closer to practical deployment in such settings. However, its safety guarantees are empirical rather than formally certified, so deploying it on physical systems without additional verification or runtime monitoring could lead to failures in situations not covered by the data.